\documentclass{article}
\usepackage{arxiv}
\usepackage{lmodern}
\usepackage{graphicx}
\usepackage{subcaption}
\usepackage{dingbat}
\usepackage{amsmath, amsthm, amssymb}
\DeclareMathOperator*{\Supp}{Supp}
\DeclareMathOperator*{\Id}{Id}
\DeclareMathOperator*{\Bary}{Bary}

\DeclareMathOperator*{\diag}{diag}
\DeclareMathOperator*{\argmin}{argmin}
\usepackage{algorithm}
\usepackage{algpseudocode}

\usepackage{wrapfig}
\usepackage{url}
\usepackage{hyperref}
\usepackage[T1]{fontenc}
\newtheorem{lemma}{Lemma}

\newtheorem{theorem}{Theorem}
\newtheorem{corollary}{Corollary}
\newtheorem{definition}{Definition}

\newcommand{\R}{\mathbb{R}}
\newcommand{\M}{\mathcal{M}}
\newcommand{\E}{\mathcal{E}}
\newcommand{\A}{\mathcal{A}}
\newcommand{\I}{\mathcal{I}}
\newcommand{\J}{\mathcal{J}}
\newcommand{\U}{\mathcal{U}}
\newcommand{\W}{\mathcal{W}}
\newcommand{\X}{\mathcal{X}}
\newcommand{\prox}{\text{prox}}
\newcommand{\proj}{\text{proj}}
\newcommand{\diff}{\:\mathrm{d}}

\usepackage{xcolor}
\usepackage[normalem]{ulem} 

\newif\ifshowcomments
\showcommentstrue      

\DeclareMathOperator*{\divx}{div_{\X}}
\date{\today}
\hypersetup{
 pdfauthor={David Gentile},
 pdftitle={Barycenters of Measures on Graphs},
 pdfkeywords={},
 pdfsubject={},
 pdfcreator={Emacs 30.2 (Org mode 9.7.34)}, 
 pdflang={English}}
\date{December 2025}

\begin{document}

\title{Static and Dynamic Approaches to Computing Barycenters of Probability Measures on Graphs}
\author{David Gentile \thanks{Tufts University, Dept. of Mathematics, Medford, MA, USA} \And James M. Murphy \(^*\)}

\maketitle

\begin{abstract} The optimal transportation problem defines a geometry of probability measures which leads to a definition for weighted averages (barycenters) of measures, finding application in the machine learning and computer vision communities as a signal processing tool. Here, we implement a barycentric coding model for measures which are supported on a graph, a context in which the classical optimal transport geometry becomes degenerate, by leveraging a Riemannian structure on the simplex induced by a dynamic formulation of the optimal transport problem.
We approximate the exponential mapping associated to the Riemannian structure, as well as its inverse, by utilizing past approaches which compute action minimizing curves in order to numerically approximate transport distances for measures supported on discrete spaces. Intrinsic gradient descent is then used to synthesize barycenters, wherein gradients of a variance functional are computed by approximating geodesic curves between the current iterate and the reference measures; iterates are then pushed forward via a discretization of the continuity equation. Analysis of measures with respect to given dictionary of references is performed by solving a quadratic program formed by computing geodesics between target and reference measures. We compare our novel approach to one based on entropic regularization of the static formulation of the optimal transport problem where the graph structure is encoded via graph distance functions, we present numerical experiments validating our approach\footnote{Library and code for reproducing experiments is available at \url{github.com/dcgentile/GraphTransportation.jl}}, and we conclude that intrinsic gradient descent on the probability simplex provides a coherent framework for the synthesis and analysis of measures supported on graphs.
\end{abstract}


\section{Introduction}

Applications and problems centering on graphs have long attracted the attention of mathematicians, dating as far back as Euler's celebrated consideration of the seven bridges of K{\"o}nigsberg \cite{imperatorskaiaakademianaukrussiaCommentariiAcademiaeScientiarum1726}. Problems formulated in the graph setting have ranged from classical combinatorial optimization considerations, such as \(k\)-coloring problems \cite{brooksColouringNodesNetwork1941} and max-flow/min-cut algorithms \cite{jrMaximalFlowNetwork1956}, to modern machine learning tools such as spectral clustering \cite{ngSpectralClusteringAnalysis2001} and graph convolutional neural networks \cite{wuGraphNeuralNetworks2022}. In recent years, signal processing on graphs has become a problem of growing interest in the applied mathematics community because graphs and discrete metric spaces arise naturally as data structures in domains such as sensor networks, opinion dissemination, and measurements of community segregation \cite{stankovicGraphSignalProcessing2019, ruizGraphonSignalProcessing2021, duchinMeasuringSegregationAnalysis2023}, where the signal domain is best represented as some discrete, potentially irregular space.  The essence of a graph is in its connectivity information --- i.e., which nodes are connected to each other --- and it therefore is natural to look for signal processing methods which incorporate that connectivity information. Study of graph connectivity finds application in several subdomains of machine learning, such as clustering methods \cite{ngSpectralClusteringAnalysis2001, roblitzFuzzySpectralClustering2013, nieConstrainedLaplacianRank2016} and supervised/semisupervised learning problems \cite{songGraphBasedSemiSupervisedLearning2023,chenGraphBasedChangePointAnalysis2023, huangLaplacianChangePoint2020}.

\begin{wrapfigure}{l}{0.60\textwidth}
    \centering
    \includegraphics[width=0.98\linewidth]{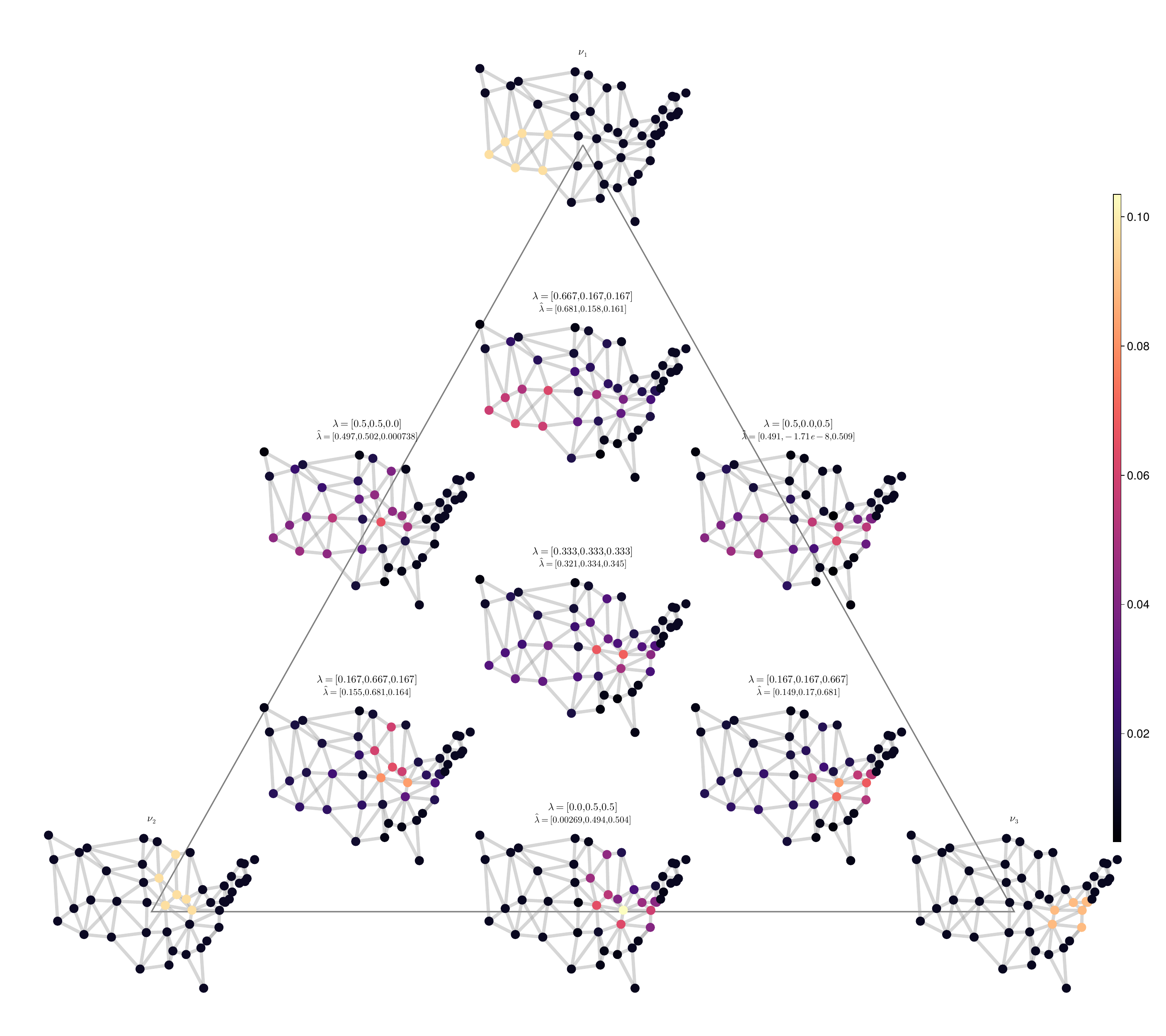}
    \caption{A schematic of the barycentric coding model for measures supported on a graph. At the extremal points of the simplex, which correspond to the vertices of the triangle, are the reference measures; away from the extremal points are \emph{barycentric interpolations} of the reference measures for certain weights \(\lambda\). The coordinates recovered by solving the analysis problem are displayed and labeled with \(\hat\lambda\).}
    \label{fig:BCMSchematic}
\end{wrapfigure}
Simultaneously, optimal transport has become a popular computational tool and attracted attention from the statistics and machine learning communities \cite{gangboShapeRecognitionWasserstein2000, arjovsky2017wasserstein, genevayLearningGenerativeModels2018, cuturiDifferentiableRankingSorting2019, cuturiSinkhornDistancesLightspeed2013,gabrielComputationalOptimalTransport2019,  rigollet2025mean}. The optimal transport problem endows the space of probability measures with a metric which encodes the geometry of the underlying space; in so many words, the distance between two distributions is the total cost it takes to transport one to the other --- this contrasts with, e.g., the \(L^2\) difference between densities of two absolutely continuous distributions, or the Kullback-Liebler divergence between them. In recent years, this has opened up new paradigms for signal processing problems \cite{schmitzWassersteinDictionaryLearning2018, muellerGeometricallyRegularizedWasserstein2023, mallery2025synthesis, werenskiLinearizedWassersteinBarycenters2025, cazelles2020wasserstein}. The barycentric coding model for probability measures \cite{bonneelWassersteinBarycentricCoordinates2016, werenskiMeasureEstimationBarycentric2022, gunsiliusTangentialWassersteinProjections2024, muellerGeometricallyRegularizedWasserstein2023, gentile2025regularized, werenskiLinearizedWassersteinBarycenters2025, mallery2025synthesis, jiangAlgorithmsMeanfieldVariational2025} is a recently developed tool for the synthesis and analysis of probability measure-valued data. The central conceit of the model, a schematic of which is shown in Figure \ref{fig:BCMSchematic}, is that the theory of barycenters in the \(L^2\)-Kantorovich metric (also known as the \(2\)-Wasserstein metric) on probability measures supported on \(\R^d\) \cite{aguehBarycentersWassersteinSpace2010} provides a synthesis mechanism for expressing the weighted average of several reference measures, provided that the weights are simplicial, i.e., that they are non-negative and sum to unity. On the other hand, an inverse problem for this weighted synthesis of measures can also be established and solved, meaning that for a given target measure and a family of reference measures, it is possible to identify an appropriate set of weights so that (re)-synthesizing the references according to these found weights results in the target measure.

The purpose of this work is to propose and implement a barycentric coding model for data which can be represented as probability measures on a graph\footnote{In this work, all measures considered are probability measures, potentially with respect to a fixed reference measure, and having total mass \(1\), although we may occasionally drop the modifier ``probability''.} . The change of setting from a continuous domain to a discrete one poses some interesting technical challenges. The \(L^2\)-Kantorovich metric becomes degenerate when the underlying metric space is discrete, and immediately renders the aforementioned work on barycenters of measures inapplicable \cite{maasGradientFlowsEntropy2011}. Nevertheless, since the early \(2010\)s, transport metrics for probability measures on graphs have been constructed and studied \cite{maasGradientFlowsEntropy2011, chowFokkerPlanckEquations2012}, and it has been shown that these discrete transport metrics do indeed bear real relation to the classically studied \(L^2\)-Kantorovich metric for measures supported on \(\R^d\) --- in particular, it is known that, for a sequence of suitably constructed graphs \& appropriately chosen admissible mean (see \ref{DOTDeets} for details on admissible means), the discrete transport metric converges in the sense of Gromov-Hausdorff to the classical \(L^2\)-Kantorovich distance on \(\R^d\), as the graphs become finer and finer approximations of Euclidean domains \cite{gladbachScalingLimitsDiscrete2020, gigliGromovHausdorffConvergenceDiscrete2013, erbarRicciCurvatureFinite2012}.

In order to achieve our implementation, we leverage the Riemannian structure of transport on discrete metric spaces, along with previously developed numerical routines for approximating geodesics in such a setting \cite{erbarComputationOptimalTransport2020}. The momentum information associated to discretized energy minimizing geodesics is interpreted as an oracle for the inverse of the exponential mapping. Then, barycenters, characterized as minimizers of a certain variance functional which is convex in the Euclidean geometry, are synthesized using an intrinsic gradient descent approach, where movement on the simplex is approximated by a continuity equation-inspired heuristic, while analysis of barycenters is performed by solving a finite dimensional quadratic program formed by approximating the inverse of the exponential mapping based at the target of the analysis. We find that our algorithms are consistent in the sense that synthesizing a barycenter of given reference measures and coordinates and then analyzing the output measure with respect to the reference measures leads to accurate recovery of the given coordinates.

\subsection{Structure of the Paper}
The structure of this paper is as follows. In Section 2, we provide the essential background information concerning graphs, optimal transport, transport on discrete metric spaces, and Riemannian optimization. In Section 3, we explicate the algorithms used to implement the barycentric coding model on graphs. In Section 4, we present our numerical results. We conclude with a discussion in Section 5, and outline future directions for improving this work both from theoretical and practical points of view.

\section{Background}
\subsection{Relation to the Literature}

Let us briefly recap the development of barycenters in the \(L^2\)-Kantorovich geometry (also known as the \(W_2\) geometry) and optimal transport methods for measures on graphs, as well as highlight some alternatives to the one which serves as our focal point. For the granular details of optimal transport, we refer the reader to any of the excellent textbooks on the subject, such as \cite{ambrosio2021lectures, figalli2021invitation, villaniOptimalTransport2009, santambrogioOptimalTransportApplied2015}; in Section \ref{OTSec} we recall the most essential definitions and results for understanding of this work.

The particular subdomain of optimal transport which will be of interest to us in this work concerns \emph{\(W_2\) barycenters}, a construction first introduced in \cite{aguehBarycentersWassersteinSpace2010}, wherein the authors demonstrate that under mild regularity assumptions on a collection of measures supported on \(\R^d\), a unique weighted center of mass can be defined for the measures.  In \cite{kimWassersteinBarycentersRiemannian2017}, it was shown that the construction can be extended to the non-Euclidean case, where measures are supported on a Riemannian manifold. \(W_2\) barycenters have since attracted much attention in the optimal transport community. Computationally, it was shown in \cite{benamou2015iterative} that entropic regularization (see Section \ref{EOTSec} for discussion of entropic regularization) allows barycenters of discrete measures to be approximated quickly via Bregman projection methods, and in \cite{alvarez-estebanFixedpointApproachBarycenters2016} a fixed-point method is applied to obtain rapid convergence to barycenters when all the measures under consideration are Gaussian. 
Outside of special cases however, \(W_2\) barycenters are generally hard to compute \cite{altschuler2022wasserstein}. \(W_2\) barycentric coordinates were introduced for histogram data in \cite{bonneelWassersteinBarycentricCoordinates2016}. A similar regression scheme exists when the measures are absolutely continuous \cite{werenskiMeasureEstimationBarycentric2022}. The statistics of \(W_2\) barycenters have also been of interest \cite{zemelFrechetMeansProcrustes2019, panaretosInvitationStatisticsWasserstein2020}. Recently stochastic gradient descent in the \(W_2\) metric has been applied to approximate \(W_2\) barycenters \cite{backhoffStochasticGradientDescent2025}. This method is similar to ours in spirit, but the distinction between continuous and discrete domains makes them essentially different in their implementation and background.  Recently a unified framework for studying first-order optimization methods in the Wasserstein space was presented in \cite{lanzettiFirstOrderConditionsOptimization2025}, yielding results on necessary and sufficient optimality conditions for constrained optimization problems in the \(W_2\) geometry, in particular justifying ``set the gradient to zero'' approaches.

In \cite{janatiDebiasedSinkhornBarycenters2020}, the smoothing bias introduced by the reference measure used to define the entropic regularization of the \(W_2\) geometry was studied resulting in the definition of the debiased Sinkhorn barycenter, as a compromise preserving the usefulness of fast Sinkhorn iterations for barycenter computations while eliminating the smoothing introduced by the regularization. The debiased Sinkhorn divergence was further explored in the context of the barycentric coding model in \cite{mallery2025synthesis}, where the synthesis and analysis problems are characterized as solutions to a fixed point of the weighted average of certain regularized displacement maps and solutions to finite-dimensional convex quadratic programs respectively. Linearization of the barycentric coding model (LBCM), where the problem is lifted into the tangent space at a fixed reference measure, was considered in \cite{merigotQuantitativeStabilityOptimal2020} as well as \cite{werenskiLinearizedWassersteinBarycenters2025}, with particular attention paid to the case of compatible measures; additionally, a suitable family of reference measures is identified such that the LBCM is capable of expressing all probability measures on the closed unit interval.   Regularity of entropy-regularized barycenters was explored in \cite{carlierEntropicWassersteinBarycentersPDE2021}, where an elliptic system of Monge-Amp{\'e}re equations is employed to establish existence and uniqueness, bounds on moments and the Fisher information of the barycenters, and to perform asymptotic analysis, obtaining a central limit theorem and law of large numbers for barycenters of empirical measures.

Also of special interest to us is recent research on optimal transportation of measures which are supported on graphs, instead of \(\R^d\). The earliest work in this direction are the approaches developed in parallel by \cite{maasGradientFlowsEntropy2011} and \cite{chowFokkerPlanckEquations2012}, wherein a Riemannian structure is constructed for the simplex, inspired by the fluid dynamics formulation of the classical optimal transport problem.   This approach has been expanded on in recent years. For example, it has been used to define a notion of discrete Ricci curvature for Markov chains \cite{erbarRicciCurvatureFinite2012}. The so-called Gromov-Hausdorff scaling limits of this construction, where the graph is interpreted as a coarse representation of some continuous metric space, have been investigated and found, under certain technical assumptions, to be consistent in the limit with the classical \(L^2\)-Kantorovich metric \cite{gigliGromovHausdorffConvergenceDiscrete2013,gladbachScalingLimitsDiscrete2020}. Attention has been given to the underlying Hamilton-Jacobi dynamics of the geodesics in this discrete transport geometry \cite{erbarGeometryGeodesicsDiscrete2019,gangboGeodesicsMinimalLength2019}, and many classical objects from Riemannian geometry, such as the Christoffel symbols and the Riemannian curvature tensor have been computed \cite{liTransportInformationGeometry2022}. We also point out that other approaches to the problem of continuous flows in the space of measures on graphs have been developed \cite{solomon2016continuous}, as well as approaches inspired by the Beckmann problem \cite{robertson2025generalization}, the notion of circuit resistance \cite{robertson2024all}, and vector-valued transport for measures supported on spaces which are products of Euclidean spaces and graphs \cite{craig2025vector}.

Finally, let us highlight a few papers in the domain of Riemannian optimization. Riemmanian optimization is a broad field of active interest to the applied mathematics community. The particular problem of finding the Riemannian center of mass via intrinsic gradient descent was considered in \cite{afsariRiemannian$L^p$Center2011}, and further investigation into the algorithm's convergence properties was made in \cite{afsariConvergenceGradientDescent2013}. 

\subsection{Common Objects}
We start by recalling some common definitions from the theories of graphs, analysis, and probability. 

A (finite) \emph{graph} (with its shortest path metric) is a quintessential example of a discrete metric space. Specifically, a graph is a pair \(G = (V,E)\), where \(V = \{1, 2, \dots, n\}\) is the vertex set and \(E \subset V \times V\) is the edge set, encoding the information of which vertices connect to each other.  For a fixed pair of vertices \(u,v\in V\), a sequence of vertices \(v_1, v_2, \dots, v_k\) is called a \(uv\) walk if \(v_1 = u, v_k=v,\) and \((v_i, v_{i+1}) \in E\) for \(i = 1, 2, \dots, k-1\). By defining
\[d(u,v) := \inf\left\{\sum\limits_{i} \mathbf{1}_{(v_i, v_{i+1})\in E} : (v_1, \dots, v_k) \text{ is a }uv \text{ walk}\right\},\]
we obtain the shortest path metric on the graph \(G\). The \emph{degree} of a node \(x \in V\) is the number of edges \(y \in V\) such that \((x,y) \in E\), and is denoted \(\deg(x)\).

For the remainder, it will be useful to represent graphs via Markov transition kernels. Note that we can always pass back and forth between Markov transition kernels and graphs. To construct a Markov transition kernel representing an arbitrary simple, connected graph \(G = (V,E)\), one takes

\begin{equation*}
  \pi(x) = \frac{\deg(x)}{|E|}; \qquad
  Q(x,y) = \begin{cases} 
    \frac{1}{|E|\pi(x)} & (x,y) \in E \\
    0 & \text{else}
  \end{cases}.
\end{equation*}
We also note that this construction can be extended to graphs equipped with symmetric weight matrices \([\omega(x,y)]_{(x,y) \in E}\) by setting \(\deg_{\omega}(x) := \sum_{y \in V} \omega(x,y)\). Then the corresponding Markov kernel is defined by 
\begin{equation*}
  \pi(x) = \frac{\deg_\omega(x)}{\sum_{y\in V}\deg_\omega(y)}; \qquad
  Q(x,y) = \begin{cases} 
    \frac{\omega(x,y)}{\deg_\omega(x)}& (x,y) \in E \\
    0 & \text{else}
  \end{cases}.
\end{equation*}
This representation will be useful both because it serves as part of the basis for the theory of transport on discrete metric spaces, and also because it opens up a new class of metrics on graphs. Because the shortest path metric does not always capture all the connectivity information of a graph, other metrics have been constructed and considered --- in particular, the parameterized class of metrics known as the diffusion distances \cite{coifmanDiffusionMaps2006, coifmanGeometricDiffusionsTool2005}.

\begin{definition}
  The squared \emph{diffusion distance} between nodes \(x,y \in V\) with time parameter \(t\) is given by
\[d_t^2(x,y) := \sum\limits_{w \in V}\frac{(p(w,t|x) - p(w,t|y))^2}{\pi(w)},\]
where \(p(w,t|x)\) denotes the probability of reaching node \(w\) in \(t\) steps starting from node \(x\).
\end{definition}
The diffusion distance defines a metric on the vertices that which measures the connectivity between a pair \(x,y \in V\) --- points connected by many short paths will be closer together in the diffusion distance \cite{coifmanGeometricDiffusionsTool2005, coifmanDiffusionMaps2006}.  Diffusion distances can be understood as a type of non-linear dimensionality reduction, and we refer the reader interested in this aspect of diffusion distances to the papers \cite{maggioni2019learning, tenenbaumGlobalGeometricFramework2000, roweisNonlinearDimensionalityReduction2000, belkinLaplacianEigenmapsDimensionality2003}. Connectivity-encoding metrics on graphs are an area of active interest, and other such metrics, such as the Fermat distances \cite{bijralSemisupervisedLearningDensity2011, mckenzie2019power, little2022balancing, chazalchoosing, trillosFermatDistancesMetric2024} might also be employed in this setting, which we leave this to future work. Later on, we will use these distances to construct cost matrices in the context of computing entropically regularized \(W_2\) barycenters on discrete spaces.

\begin{definition}
  Take \(p\) to be a positive integer. The \emph{\(p\)-simplex}, denoted \(\Delta^{p-1}\) is the set of vectors in \(\R^p\) with non-negative components which sum to 1, in other words
  \begin{equation*}
   \Delta^{p-1} = \left\{\lambda = (\lambda_1, \dots, \lambda_p) : \lambda_i \ge 0, \sum_i \lambda_i = 1\right\} .
  \end{equation*}
\end{definition}

We denote the set of non-negative measures \(\nu\) with total mass of \(1\) defined on a measurable space \(\Omega\) by \(\mathcal{P}(\Omega)\). If \(\Omega = \R^d\), then we denote by \(\mathcal{P}_{2}(\Omega)\) the set of probability measures \(\nu\) such that \(\int_\Omega \|x\|^2\diff\nu(x) < \infty\). If there exists a density function \(\rho(x): \Omega \to \R\) such that for all Lebesgue measurable \(A \subset \Omega\) it holds that \(\nu(A) = \int_A\rho(x)\diff x\), then \(\nu\) is said to be absolutely continuous, and belongs to the set \(\mathcal{P}_{2,ac}(\Omega).\)

\begin{definition}
  Let \(\mu, \nu \in \mathcal{P}(\Omega)\). A probability measure \(\gamma \in \mathcal{P}(\Omega\times\Omega)\) is said to be a \emph{coupling} of the measures \(\mu\) and \(\nu\) if for all Borel measurable sets \(A\), the following identities hold:
\begin{equation*}
\gamma(A\times\Omega) = \mu(A); \quad \gamma(\Omega\times A) = \nu(A).
\end{equation*}
 The set of all couplings of \(\mu\) and \(\nu\) is denoted \(\Pi(\mu, \nu)\).
\end{definition}
Note that \(\Pi(\mu,\nu)\) is always non-empty because the product measure \(\mu \otimes \nu\) is itself a coupling. In the context of finite discrete measures, couplings may be represented as matrices satisfying linear constraints. If \(T: \R^d \to \R^d\) is a measurable mapping, and \(\gamma \in \Pi(\mu,\nu)\) has the form \((\Id \times T)\# \mu\), where \(\#\) is the measure push-forward operator, then we say that the plan \(\gamma\) is induced by mapping \(T\), and we call \(T\) a \emph{transport map} taking \(\mu\) to \(\nu\).

\subsection{Transport in a Euclidean Domain}
\label{OTSec}
\subsubsection{Overview}
Conceptually, the \emph{optimal transportation problem} is this: suppose we have a pile of mass that needs to be moved from one location to a new one, and that transporting a unit of mass from point \(x\) to point \(y\) has associated cost \(c(x,y)\) --- how can the mass be transported in a cost minimizing way, and what is the minimum cost associated to the optimal method? As it happens, this intuitive optimization problem is the genesis for a geometry of probability measures encoding the structure of the underlying space on which those measures are defined. 

The essential things about optimal transport that we need to understand are: that optimal transport endows the space of probability measures with a geometry; that geodesic curves in the geometry of transport obey the continuity equation and in particular minimize its action; that barycenters are defined in terms of a geometric minimization problem.

\subsubsection{Static Formulation, Kantorovich Problem, and Brenier's Theorem}
Let us recall the basic building blocks for the \(L^2\)-Kantorovich geometry of probability measures. The theory of optimal transportation in its modern formulation is centrally concerned with the problem of solving a certain linear program with constraints defined in terms of measures. Let \(\mu, \nu \in \mathcal{P}(\R^n)\) and let \(c(x,y): \R^n \times \R^n \to \R^+\). Then the optimal transport cost for \(\mu\) and \(\nu\) with cost \(c\) is given by
\begin{equation}
\label{OTcost}
  OT(\mu,\nu) := \inf\left\{ \int_{\Omega\times\Omega}c(x,y) \diff\gamma : \gamma \in \Pi(\mu, \nu) \right\}.
\end{equation}

The minimization problem in \eqref{OTcost} is referred to as the Kantorovich problem. In general, the cost function \(c(x,y)\) can take many forms, but the case which has proven most well-studied has turned out to be when \(c(x,y) = \frac{1}{2}d^2(x,y)\).  For the case of \(\Omega = \R^d\), we denote \(d(x,y)\) by \(\|x - y\|\); norms in other Hilbert spaces are differentiated by subscript, e.g. \(\|\cdot\|_H.\) Taking \(\Omega = \R^d\), this is the setting of Brenier's theorem \cite{brenierPolarFactorizationMonotone1991}, one of the foundational results of the modern field, which we reproduce here.

\begin{theorem}[\cite{brenierPolarFactorizationMonotone1991}, Theorem 1.3]
  \label{Brenier}
Let \(\mu, \nu\) be probability measures on \(\R^d\) satisfying \linebreak \(\int_{\R^n} \|x\|^2 \diff \mu(x), \int_{\R^n}\|y\|^2 \diff\nu(y) < \infty\). Let \(c(x,y) = \frac{1}{2}\|x-y\|^2\), and suppose that \(\mu\) gives no mass to \((d-1)\) surfaces of class \(C^2\). Then there exists a \(\mu\)-a.e. unique optimal transport map \(T\) from \(\mu\) to \(\nu\), and a convex function \(\phi\) such that \(T = \nabla\phi\).
\end{theorem}

The convex function \(\phi\) is known as the \emph{Brenier potential} and \(T\) is referred to as the \emph{optimal transport (OT) map}. It can be shown that in this case, and under the further stipulation that space of probability measures is restricted to the set \(\mathcal{P}_2(\Omega)\) of those measures with finite second moment, then \eqref{OTcost} defines a metric on \(\mathcal{P}_2(\Omega)\), known as the \(2\)-Wasserstein or \(L^2\)-Kantorovich metric and denoted by \(W_2\) \cite{santambrogioOptimalTransportApplied2015}.

Note also some other foundational work in the modern field of transportation: in \cite{smithNoteOptimalTransportation1987} the cyclical monotonicity characterization of optimal plans is established; in \cite{mccannConvexityPrincipleInteracting1997}, what is now known as the \emph{McCann interpolant} was introduced, which describes the dynamic transportation of one distribution to another, achieved through pushing the initial distribution forward by the linear interpolation of the optimal transport map and the identity function; in  \cite{gangboGeometryOptimalTransportation1996}, existence and uniqueness for the optimal transport problem are established for a broad class of cost functions both convex and concave.

The Kantorovich problem can be thought of as a static formulation of the optimal transportation problem --- the optimal coupling \(\gamma^*\) describes the measure of mass moving from some set \(A\) to some set \(B\), but \emph{a priori} gives no information about the path the mass takes when moving between the sets.  As we will emphasize below, the dynamic formulation of transportation will be of the utmost importance for our work.

\subsubsection{Dynamic Formulation and Benamou-Brenier Theorem}
It is also a remarkable fact that the \(W_2\) metric is formally similar to the kinds of metrics which arise in Riemannian geometry. This fact is expressed by the Benamou-Brenier formula reproduced below. In order to provide necessary context for this formula, first recall the \emph{continuity equation with temporal boundary conditions}. A function \(\rho:[0,1] \times \R^d\) and a vector field \(V:[0,1]\times\R^d\) are said to satisfy the continuity equation with temporal boundary data \((\rho_0, \rho_1)\) if the following holds:

\begin{equation}
  \label{CCE}
 \frac{\partial\rho}{\partial t} + \nabla\cdot(\rho V) = 0; \quad \rho(0, x) = \rho_0(x), \quad \rho(1,x) = \rho_1(x).
\end{equation}

On a physical level, this equation is asserting that a quantity of mass is evolving continuously from distribution \(\rho_0\) to distribution \(\rho_1\). The Benamou-Brenier formula connects this fluid dynamics problem to the optimal transportation problem by asserting that the optimal transport cost for measures \(\rho_0\) and \(\rho_1\) with ground cost given by \(\|\cdot - \cdot\|^2\) is exactly the minimum energy required to move the mass continuously from one configuration to another.

\begin{theorem}[\cite{benamouComputationalFluidMechanics2000}, Proposition 1.1]
  \label{BBF}
  Given absolutely continuous probability measures on \(\R^d\) with finite second moment \(\rho_0, \rho_1\), the square of the \(L^2\)-Kantorovich distance satisfies
  \begin{equation*}
	W_2^2(\rho_0,\rho_1) = \inf \left\{ \int_{0}^1\int_{\R^d}\|V(t,x)\|^2\rho(t,x) \diff x \diff t \right\},
  \end{equation*}
where the infimum is taken over all pairs \((\rho, V)\) satisfying the continuity equation with boundary data \(\rho_0, \rho_1\).
\end{theorem}

To see how this can be interpreted as a Riemannian metric, notice that the integrand might as well be written
\begin{equation}
  W_2^2(\rho_0, \rho_1) = \inf \left\{ \int_0^1 \int_\Omega \rho(t,x)\|V(t,x)\|^2 \diff x \diff t \right\} = \inf \left\{ \int_0^1 \|V(t,x)\|_{L^2(\rho(t,x))}^2 \diff t \right\}.
\label{BBEQ}
\end{equation}
The RHS can now be read as the minimization of an energy functional, with the minimization occurring over the space of all absolutely continuous curves with endpoints \(\rho_0\) and \(\rho_1\). This point will be important in the development of a transport-like metric for probability measures supported on graphs. We want to emphasize here that the Benamou-Brenier formula gives a \emph{dynamic} interpretation of the optimal transportation problem --- the minimizer of \eqref{BBEQ} will be a curve in the space of measures \(\rho(t)\), which can be interpreted as tracking the instantaneous distribution of mass at each time \(t \in [0,1]\) as the distribution evolves from \(\rho_0\) to \(\rho_1\). This is in contrast to the static formulation given by the Kantorovich problem.

The energy minimizing curve may also be connected to the static plans and maps formulation of Kantorovich via the McCann interpolation. In \cite{mccannConvexityPrincipleInteracting1997}, it is shown that if \(T^*\) is the optimal transport map pushing \(\rho_0\) to \(\rho_1\), then the interpolation, defined by the expression
\begin{equation*}
    \rho_t := ((1-t)\Id + tT^*)\#\rho_0,
\end{equation*}
is energy minimizing over the interval \([0,1]\).

With this interpretation of the \(L^2\)-Kantorovich metric in mind, a natural question to ask is the following: is there a corresponding interpretation of the notion of a tangent space at a point? The answer is yes, and, following \cite{ambrosioGradientFlowsMetric2005}, we give the following definition:
\[T_\nu \mathcal{P}_{2,ac}(\R^n) := \left\{ \beta(\nabla\phi - \Id)  : \beta > 0, \phi \in C_c^\infty(\R^n), \phi \text{ convex }\right\},\]
where the closure is in \(L^2(\nu)\). If \(u,v\) are two tangent vectors in \(T_\nu \mathcal{P}_{2,ac}(\R^n)\), then their inner product is defined by multiplying them together and integrating against the measures \(\nu\), \(\langle u,v\rangle_\nu := \int_{\R^d}\langle u,v\rangle \diff \nu(x)\). The existence of this formal Riemannian interpretation will be important because it opens the door to applying first-order optimization methods for minimizing functionals defined on \((\mathcal{P}_{2,ac}(\Omega), W_2)\).

In full generality, absolute continuity of the boundary measures is not necessary for (\ref{BBF}) to hold. In \cite{ambrosioGradientFlowsMetric2005}, a rigorous investigation is undertaken to show that in a setting as general as \(\Omega\) being a separable Hilbert space, there is a direct correspondence between absolutely continuous curves \(\nu_t:(a,b) \to \mathcal{P}(\Omega)\) and solutions \emph{in the sense of distributions} to the continuity equation. Taking \(\Omega=\R^d\), a curve \(\nu_t\) of Borel probability measures and a time-dependent vector field \(v_t\) are said to satisfy the continuity equation in the sense of distributions if the following holds for all test functions \(\varphi \in C_{c}^\infty(\R^d \times (a,b))\):

\begin{equation*}
  \int_{a}^{b}\int_{\R^d} \left( \partial_t \varphi(x,t) + \langle v_t(x), \nabla_x\varphi(x,t) \rangle \right) \diff \nu_t(x) \diff t = 0.
\end{equation*}

This weak formulation allows for a class of admissible solutions which are lacking in the regularity that would normally be required to solve the equation \eqref{CCE} in the classical sense. This will be important later on when finding weak solutions to a modification of \eqref{CCE} is of use for approximating the discretized transport distance for measures on graphs (see Section \ref{DOTsec}).
\subsubsection{Barycenters of Probability Measures and the Barycentric Coding Model}
\begin{wrapfigure}{l}{0.4\textwidth}
    \centering
    \includegraphics[width=0.4\textwidth]{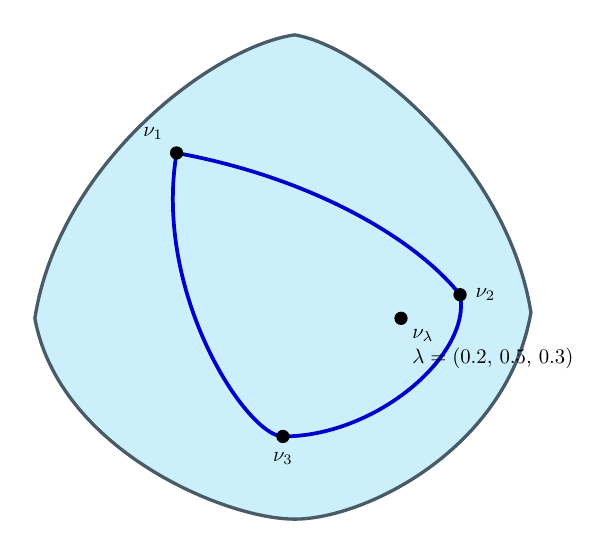}
    \caption{Weighted center of mass on a non-Euclidean space; \(\nu_\lambda\) is the synthesis of the points \(\nu_1, \nu_2, \nu_3\) with weights \(\lambda = (0.2,0.5,0.3)\)}
    \label{fig:ManifoldBarycenter}
\end{wrapfigure}
Now with a metric in hand, we would like to consider extending the notion of convex combinations of points to this setting.  An example of a weighted center of mass on a non-Euclidean domain is show in Figure \ref{fig:ManifoldBarycenter}.  In a Euclidean space, given points \(\left\{ x_i \right\}_{i=1}^p\) and \(\lambda \in \Delta^{p-1}\), the convex combination \(\sum \lambda_i x_i\) minimizes the function \(\sum \lambda_i \|x_i - x\|^2\). Replacing \(\|\cdot - \cdot\|^2\) by \(d^2(\cdot,\cdot)\) suggests how convex combinations might be reasonably defined in spaces lacking linear structure, such as \((\mathcal{P}_{2}(\Omega), W_2)\).
In \cite{aguehBarycentersWassersteinSpace2010}, the minimization problem
\begin{equation}
  \label{VF}
  \inf_\nu \left\{J(\nu) := \sum\limits_{i=1}^p \lambda_i W_2^2(\nu_i, \nu)\right\}
\end{equation}
is studied, wherein the coefficients \((\lambda_1, \dots, \lambda_p)\) are real numbers such that \(\lambda_i \ge 0\) for each \(i=1,\dots,p\), and \(\sum\lambda_i=1\), and where the \(\nu_i\) are probability measures with finite second moment defined on \(\R^n\). The fixed measures \(\{\nu_i\}_{i=1}^p\) are referred to as the \emph{reference measures}, and the coefficients \(\left\{ \lambda_i \right\}_{i=1}^p\) are the \emph{weights}. We refer to the functional \(J\) as the \emph{variance functional}. Minimizers \(\nu^*\) of \eqref{VF} are referred to as \emph{barycenters} of the reference measures \(\left\{ \nu_i \right\}_{i=1}^p\) with weights \(\left\{ \lambda_i \right\}_{i=1}^p\), and we write \(\nu^*=\Bary(\left\{ \nu_i \right\}, \lambda)\).  Recall that for a Hilbert space \(H\) and convex function \(\varphi: H \to \R\), the \emph{Fenchel conjugate} of \(\varphi\) is denoted \(\varphi^*\) and defined by the expression:
\begin{equation*}
  \varphi^*(y) := \sup\limits_{x} \langle x, y \rangle_{H} - f(x).
\end{equation*}
A major result of \cite{aguehBarycentersWassersteinSpace2010} is the following:

\begin{theorem}(\cite{aguehBarycentersWassersteinSpace2010}, Proposition 3.8)
  \label{ACTheorem}
  Let \(\mathcal{M}(\R^d)\) be the space of bounded Radon measures on \(\R^d\), let \(\mathcal{M}_{+}^1(\R^d)\) be the set of Radon probability measures on \(\R^d\), and let \[X' = \left\{ \mu\in \mathcal{M}(\R^d) : (1 + \|x\|^2)\mu \in \mathcal{M}(\R^d) \right\}.\] Assume that \(\nu_i\) vanishes on sets of Hausdorff dimension \(\le d-1\) for every \(i=1,\dots,p\), and let \(\nu \in X' \cap \mathcal{M}_+^1(\R^n)\). Then the following conditions are equivalent:
  \begin{enumerate}
    \item \(\nu\) solves \eqref{VF}.
    \item There exist convex potentials \(\psi_i\) such that \(\nabla\psi_i\) is Brenier's map transporting \(\nu_i\) to \(\nu\), and a constant \(C\) such that \(\sum\limits_{i=1}^p\lambda_i\psi_i^*(y)\le C + \frac{\|y\|^2}{2}, \forall y \in \R^n,\) with equality \(\nu\)-a.e. Furthermore, the potentials \(\psi_i^*\) are differentiable on \(\Supp(\nu)\) and satisfy \(\sum\limits_{i=1}^p\lambda_i\nabla\psi_i^* = \Id\).
  \end{enumerate}

\end{theorem}

The last piece of background from the theory of transport in the Euclidean setting is the \emph{barycentric coding model for measures}. Fix a family of reference measures \(\left\{ \nu_i \right\}_{i=1}^p \subset \mathcal{P}_{2,ac}(\R^n)\), and let \(\nu_\lambda\) be the minimizer of \eqref{VF} for a given \(\lambda \in \Delta^{p-1}\). Then we denote by \(\Bary( \left\{ \nu_i \right\}_{i=1}^p) = \left\{ \nu_\lambda : \lambda \in \Delta^{p-1}\right\}\). The barycentric coding model consists of two problems: \emph{synthesis}, the problem of finding \(\nu_\lambda\) for given \(\lambda \in \Delta^{p-1}\), and \emph{analysis}, the problem of finding \(\lambda\) for given \(\nu_\lambda\). The content of Theorem \ref{ACTheorem} is that the synthesis problem is well-defined, provided certain regularity assumptions hold. In \cite{werenskiMeasureEstimationBarycentric2022}, it is shown under mild regularity assumptions on the family of reference measures that the analysis problem can also be solved.
\begin{theorem}[\cite{werenskiMeasureEstimationBarycentric2022}, Proposition 1]
  Let \(\left\{ \nu_i \right\}_{i=1}^p\) be a collection of suitably regular probability measures on \(\Omega\). Then \(\nu_0 \in \Bary(\left\{ \nu_i \right\}_{i=1}^p)\) if and only if
  \begin{equation}
	\min\limits_{\lambda\in\Delta^{p-1}}\lambda^TA\lambda=0,
  \label{AP}
  \end{equation}
  where, letting \(T_i\) be the optimal transport map from \(\mu_0\) to \(\mu_i\), \(A \in \R^{p\times p}\) is defined by
  \begin{equation}
	A_{ij} = \langle T_i - \Id, T_j - \Id\rangle_{L^2(\nu_0)}.
  \label{GramMatrix}
  \end{equation}
  Furthermore, if the minimum value of \(\lambda^TA\lambda\) is \(0\), and \(\lambda^*\) is an optimal argument, then \(\nu_0 = \nu_{\lambda^*}\).
\end{theorem}

Thus, if we can approximate the matrix \(A\), and solve the quadratic program \eqref{AP}, we can solve the analysis problem for the barycentric coding model; the geometry of this result is illustrated in Figure \ref{fig:ManifoldAnalysis}.  We note that this assumes perfect access to the probability measures under consideration; estimation of transport maps via i.i.d. samples from the underlying measures enables estimation of the optimal $\lambda$ in this statistical setting.

\begin{wrapfigure}{r}{0.4\textwidth}
    \centering
\includegraphics[width=0.4\textwidth]{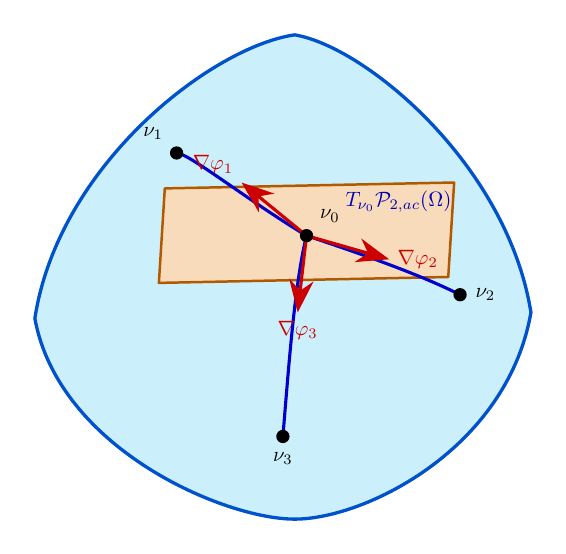}
    \caption{To solve the analysis problem, we compute the tangent vectors associated to the unit time geodesic curves connecting the target measure to each reference measure, then form the associated Gram matrix.}
    \label{fig:ManifoldAnalysis}
\end{wrapfigure}

\subsection{Entropic Regularization of the Optimal Transport Problem and Sinkhorn's Algorithm}
\label{EOTSec}
Up to this point, reference has been made frequently to regularity of measures --- but in applications, typically one works with empirical measures based on some set of observations. In the case of finitely supported measures the Kantorovich problem may be optimized via standard techniques in convex programming. However, linear programs exhibit poor scaling in the size of that support (\(\mathcal{O}(n^3\log n)\)) \cite{peleFastRobustEarth2009}. In \cite{cuturiSinkhornDistancesLightspeed2013}, it was shown that transport distances of empirical measures, i.e., normalized superpositions of Dirac point masses, could be approximated more quickly by appending a regularization term. Let \(\hat{\mu} = \sum\limits_{i=1}^n \alpha_i \delta_{x_i}, \hat{\nu} = \sum\limits_{i=1}^m\beta_i \delta_{y_i}\), be finitely supported measures, fix \(\varepsilon_{reg} > 0\), define the pairwise distance matrix \(C = [C_{i,j}]_{i=1, j=1}^{n,m}\) via
\begin{equation*}
  C_{i,j} = \|x_i - y_j\|^2,
\end{equation*}
and for \(\gamma\in\Pi(\hat{\mu}, \hat{\nu})\) define the entropy of \(\gamma\) by 
\begin{equation*}
  H(\gamma) := -\sum\limits_{i,j}\gamma_{i,j}(\log\gamma_{i,j} - 1).
\end{equation*}
Then the entropic optimal transport cost for \(\hat{\mu}\) and \(\hat{\nu}\) is defined via
\begin{equation}
  \label{EOTCost}
  OT_{\varepsilon_{reg}}(\hat{\mu}, \hat{\nu}) := \inf \left\{ \langle\gamma,C\rangle + \varepsilon_{reg} H(\gamma) \right\}.
\end{equation}
\begin{figure}
    \centering
    \includegraphics[width=0.95\textwidth]{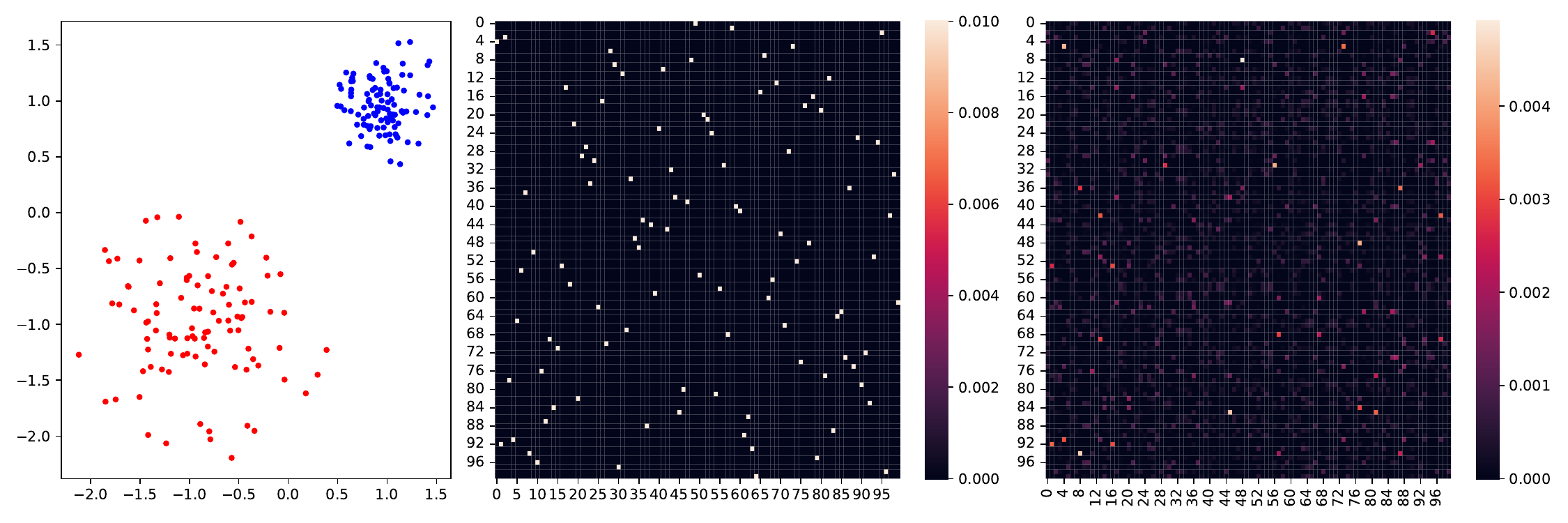}
    \caption{Comparison between a classic optimal transport plan and an entropically regularized one. Left: two samples of \(100\) points each in the plane (normally distributed with different means and variances). Center: optimal assignment between points to minimize transport cost, obtained via computation of a linear program. Right: entropic transport plan for minimizing entropic transport cost. Notice the difference in sparsity between the two plans: entropic transport results in a diffusive effect.}
    \label{fig:EntropicPlans}
\end{figure}

In the limit as \(\varepsilon_{reg} \to 0^+\), we have \(OT_{\varepsilon_{reg}}(\hat{\mu}, \hat{\nu}) \to OT(\hat{\mu}, \hat{\nu})\) \cite{gabrielComputationalOptimalTransport2019}. The effect of the regularization is demonstrated in Figure \ref{fig:EntropicPlans} --- the optimal plan in the unregularized case is a pure permutation matrix, assigning mass to move from one specific point to another in a cost-minimizing way, while the regularization introduces a diffusive effect, resulting in a much less sparse plan.
Computing the entropic optimal transport cost is accomplished via the Sinkhorn algorithm \cite{sinkhornConcerningNonnegativeMatrices1967, sinkhornDiagonalEquivalenceMatrices1967}, which converges to an optimal solution of \eqref{EOTCost} in an amount of time which is roughly quadratic in the number of samples \cite{cuturiSinkhornDistancesLightspeed2013}. For this reason, entropic regularization has become a popular computational technique in the optimal transport community \cite{genevayLearningGenerativeModels2018, mena2019statistical, cuturiDifferentiableRankingSorting2019, pooladian2021entropic, masudMultivariateSoftRank2023, mallery2025synthesis, werenskiOptimalTransportMaps2024, montesumaRecentAdvancesOptimal2025}.

\subsection{Optimal Transport on Discrete Spaces}
\label{DOTDeets}
\subsubsection{Definitions and Calculus for Discrete Metric Spaces}
We now describe how the theory of transportation differs when the underlying metric space is discrete. A \emph{(finite) discrete metric space} is a pair \((\X, d)\) where \(\X = \{1,2,\dots,n\}\), with \(d: \X\times\X\to\R^{+}\) satisfying the classic properties of a metric (symmetry, positive definiteness, and the triangle inequality).

Now, a probability measures on a graph always has a simple structure, namely it is a non-negative function \(p:V\to\R^+\) such that \(\sum p_i = 1\). It is easy to see that the space of functions on the vertex set of \(G\) is identified with \(\R^{|V|}\) and hence that the set of probability measures is the \(|V|\)-simplex.

In the Euclidean setting, the approach to the optimal transport problem based on the Kantorovich formulation and transportation plans, as opposed to transportation maps, arose in part because it was easier to treat analytically. While the problem remains tractable when the underlying domain of the measures is discrete, certain technical issues arise that make it unsuitable as a metric on the simplex. In particular, one can show that if \(\X_2 = (\{0,1\}, \{(0,1)\})\) is the simple two point graph, with two vertices connected by an undirected edge, and \(W_2\) is the classic \newline \(L^2\)-Kantorovich distance, then as metric spaces \((\mathcal{P}(\X_2), W_2) \simeq ([0,1], \sqrt{|x-y|})\). This is problematic, since the space on the RHS admits no non-trivial geodesics; in particular, every geodesic curve is \(2\)-H{\"o}lder continuous, and hence constant. Thus there can be no theory of gradient flows or of variance minimizing measures \cite{maasGradientFlowsEntropy2011}.

To remedy this issue, a construction based on the Benamou-Brenier formula was proposed in \cite{maasGradientFlowsEntropy2011}. The key to the construction is adapting the continuity equation to the setting of a discrete metric space, defining an action functional for curves in the \(|V|\)-simplex, and declaring the squared distance between two measures \(\rho_A\) and \(\rho_B\) to be the minimal action of a curve \(\rho(t)\) where \(\rho(0) = \rho_A,\) and \(\rho(1) = \rho_B\), where the minimization occurs over the space of all curves satisfying the adapted continuity equation with prescribed temporal boundary conditions. 

The construction in \cite{maasGradientFlowsEntropy2011} is phrased in terms of Markov kernels on finite sets. Let \(\X\) be a finite set, let \(Q: \X \times \X \to [0, \infty)\) be the transition rate matrix of a continuous time irreducible Markov chain on \(\X\), and let \(\pi\) be the (unique) stationary distribution of the process, satisfying \(\pi: \X \to (0,1), \sum_i \pi(i) = 1\). Assume also that \(Q\) and \(\pi\) satisfy a detailed balance equation, meaning that \(\pi(i)Q(i,j) = \pi(j)Q(j,i)\) for all \(i,j \in \X\). The set of probability measures \emph{with respect to \(\pi\)} is denoted
\begin{equation*}
  \mathcal{P}(\X) := \left\{ \rho:\X\to\R_0^+ : \sum\limits_{i=1}^{|V|} \rho(i) \pi(i) = 1\right\}.
\end{equation*}

In order to describe the adaptation of the continuity equation to the discrete settings, we first need to recall the tools of discrete calculus. Let \(\X = (V,E)\) and \(d\) be the shortest path metric. A function on (the nodes of) \(\X\) is identified with a vector \(f \in \R^{|V|}\). A vector field on (the edges of) \(\X\) is a function \(m \in \R^{|E|}\). The spaces \(\R^{|V|}\) and \(\R^{|E|}\) are equipped with the respective inner products
\begin{equation*}
  \langle\varphi,\psi\rangle_\pi = \sum\limits_{x \in V} \varphi(x)\psi(x)\pi(x); \quad \langle\Phi,\Psi\rangle_Q = \frac{1}{2}\sum\limits_{x,y\in V}\Phi(x,y)\Psi(x,y)Q(x,y)\pi(x).
\end{equation*}
The graph gradient of a function \(f\) is a vector field \(\nabla_\X f\), and is defined by
\begin{equation*}
  \nabla_{\X}f(x,y) = f(x) - f(y).
\end{equation*}
The graph divergence of a vector field \(m\) is a function \(\divx m\), defined by
\begin{equation*}
  \divx m(x) = \frac{1}{2}\sum_{y \in V} (m(y,x) - m(x,y))Q(x,y).
\end{equation*}
The graph Laplacian of a function \(f\) is the composition of the graph gradient and graph divergence, and is denoted \(\Delta_{\X}f\).  The operator \(\Delta_{\X}\) admits the explicit form
\begin{equation*}
    \Delta_{\X}f(x) = \sum\limits_{y\in\X} Q(x,y)(f(y) - f(x)) = ((Q - D)f)(x),
\end{equation*}
where \(D := \diag\left(\sum_{y \in \X}Q(x,y)\right)_{x\in \X}\). The graph Laplacian has been an object of intense study and finds applications in network cluster identification, vertex dimensionality reduction, and Fourier analysis on graphs; for details, we refer the reader to the surveys \cite{merrisSurveyGraphLaplacians1995, merris1994laplacian, stankovicGraphSignalProcessing2019}.

In the context of graph calculus, there is no way to multiply a vector field by a function and obtain another vector field. Therefore, if the continuity equation \eqref{CCE} is to be made legible in the context of a discrete metric space, some choice must be made about how to ``scale'' a vector field. To that end, an \emph{admissible mean} \(\theta(\cdot,\cdot): \R^+ \times \R^+ \to \R^+\) is chosen. For technical reasons, there are certain properties that \(\theta\) will need to satisfy: it must be concave, lower semi-continuous, \(1\)-homogeneous, symmetric, and \(C^2\) on \((0, \infty)^2\). The choices most often considered in the literature are the logarithmic mean and the geometric mean, respectively given by
\begin{equation*}
  \theta_{\log}(s,t) = \frac{s - t}{\log(s) - \log(t)}; \quad \theta_{geom}(s,t) = \sqrt{st}.
\end{equation*}
In our implementation, we prefer \(\theta_{geom}\) simply because it is easier to work with on a practical level. Once \(\theta\) is chosen, then for any measure \(\rho \in \mathcal{P}(\X)\), define the vector field \(\theta(\rho)(i,j) := \theta(\rho(i),\rho(j))\). Then the set of solutions to the \emph{discrete continuity equation} is defined by

\begin{equation*}
  \mathcal{CE}(\rho_A, \rho_B) := \left\{ (\rho, m) : \frac{\partial\rho}{\partial t} + \divx(\theta(\rho) \odot m) = 0, \: \rho(0, \cdot) = \rho_A, \: \rho(1, \cdot) = \rho_B \right\},
\end{equation*}
where \(\odot\) means componentwise multiplication. Then, if \(\rho:[0,1] \to \R^\X\) and \(m:[0,1]\to \R^{\X\times\X}\), the action of the pair \((\rho, m)\) is defined to be

\begin{equation*}
  \A(\rho, m) := \frac{1}{2} \int_0^1 \sum\limits_{i,j \in \X} \frac{m(t,i,j)^2}{\theta(\rho(t,i), \rho(t,j))}Q(i,j)\pi(i) \mathrm{d}t.
\end{equation*}
Taking \(\I_{\mathcal{CE}(\rho_A, \rho_B)}(\rho,m)\) to be \(0\) for \((\rho,m) \in \mathcal{CE}(\rho_A, \rho_B)\) and \(+\infty\) otherwise, the energy of the pair \((\rho,m)\) is defined to be
\begin{equation*}
  \E(\rho,m) = \A(\rho,m) + \I_{\mathcal{CE}(\rho_A, \rho_B)}(\rho, m),
\end{equation*}
and finally the metric \(\W\) is defined by

\begin{equation*}
  \W(\rho_A, \rho_B) := \inf \left\{ \sqrt{\E(\rho, m)} : (\rho, m) \in \mathcal{CE}(\rho_A, \rho_B) \right\}.
\end{equation*}
For full details of the construction and proof that it does indeed define a proper Riemannian structure on the interior of the simplex, we refer the reader to \cite{maasGradientFlowsEntropy2011}. Two things which are important to note concerning the discrete transport geometry. Firstly, the tangent space has an interpretation analogous to the one given in the Euclidean setting, which is to say that vectors in the tangent space at a point can are in correspondence with the set of graph gradients of functions on the graph,
\[T_{\nu}\mathcal{P}(\X) = \left\{\nabla\varphi \in \R^{|V|\times|V|} : \varphi \in \R^{|V|}\right\}.\] 
Second, that the metric tensor has an explicit expression in terms of the admissible mean \(\theta\) --- this is key, because it allows us to form a the Gram matrix needed for solving the analysis problem. We have that for \(\nabla\varphi, \nabla\psi \in T_{\nu}\mathcal{P}(\X)\)
\[\langle\nabla\varphi,\nabla\psi\rangle_{\nu} = \frac{1}{2}\sum\limits_{x,y \in V}\theta(\nu(x),\nu(y))\nabla\varphi(x,y)\nabla\psi(x,y).\]
Lastly, and perhaps most important from a theoretical perspective, is that as in the case with the \(W_2\) metric for probability measures on \(\R^d\), \(\W^2\) is convex with respect to linear interpolation.

\begin{theorem}[\cite{erbarRicciCurvatureFinite2012}, Proposition 2.11]
    For \(\tau \in [0,1]\) and \(i,j=0,1\), take \(\rho_i^j \in \mathcal{P}(\X)\) and set \(\rho_i^\tau := (1-\tau)\rho_i^0 + \tau\rho_i^1\). Then
    \[\W^2(\rho_{0}^\tau, \rho_1^{\tau}) \le (1- \tau)\W^2(\rho_0^0, \rho_1^0) + \tau\W^2(\rho_0^1,\rho_1^1).\]
\end{theorem}

\begin{corollary}
    Fix \(\{\nu_i\}_{i=1}^p \subset \mathcal{P}(\X)\), let \(\lambda\in\Delta^{p-1}\), and define the variance functional \newline \(\J: \mathcal{P}(\X) \to \R\) via
    \begin{equation}
    \label{GraphVF}
        \J[\nu] = \sum\limits_{i=1}^p \frac{\lambda_i}{2}\W^2(\nu,\nu_i).
    \end{equation}
    Then there exists \(\nu^* \in \mathcal{P}(\X)\) which minimizes \(\J\).
\end{corollary}
\begin{proof}
    Convexity of the functional \(\J\) implies that is continuous in the Euclidean geometry and since \(\J\) is defined over a compact domain, there therefore exists \(\nu^* \in \mathcal{P}(\X) \cong \Delta^{|V|-1}\) minimizing \(\J\) by the extreme value theorem.
\end{proof}

\subsection{Intrinsic Gradient Descent}
In order to compute minimizers to the variance functional, we take a perspective on the problem originating in Riemannian optimization. One of the key tools we will need for this approach is the \emph{exponential map}, a classical object in Riemannian geometry \cite{tuDifferentialGeometry2017}. 
Let \((\M, g)\) be a Riemannian manifold, let \(p \in \M\), and let \(T_p\M\) be the tangent space to \(\M\). The exponential map at \(p\) is a function 
\begin{equation*}
    \exp_p:\U\subset T_p\M \to\M, 
\end{equation*}
which maps tangent vectors at \(p\) to points in \(\M\) based on the geodesics emanating from \(p\).  If \(\gamma_{p,v}:[0,1]\to\M\) is a geodesic curve such that \(\gamma(0) = p, \gamma'(0) = v\), then the exponential map is defined by the relation
\begin{equation*}
    \exp_p(v) = \gamma_{p,v}(1).
\end{equation*}
Or in physical terms, it is the point in \(\M\) reached after unit time by departing from \(p\) in the direction \(v\). Typically the exponential map is difficult to work with, because it is defined in terms of an ordinary differential equation and the Levi-Civita connection associated to the metric \(g\). Nevertheless, the exponential map is the key ingredient for what we shall refer to as \emph{intrinsic gradient descent}. 

Now the classical gradient descent scheme in the Euclidean setting is defined by the following relation: if \(F\) is a smooth function and we wish to find a (perhaps local) minimizer of \(F\), we compute the iterates
\begin{equation*}
    x_{k+1} = x_k - t_k \nabla F(x_k),
\end{equation*}
for some positive sequence \(\left\{ t_k \right\}\).  Lacking a linear structure, this equation becomes meaningless on a Riemannian manifold, but the essential idea of walking ``in the direction of steepest descent'' remains legible, provided the exponential map is well-defined. In that case, we write
\begin{equation*}
    x_{k+1} = \exp_{x_k}\left(-t_k\nabla_{\M}F|_{x_k}\right),
\end{equation*}
where \(\nabla_\M\) is the Riemannian gradient, i.e., the unique vector in \(T_{x_k}\M\) which is dual to the differential \(\mathrm{d}F\). The intrinsic gradient of the functional \eqref{GraphVF} is actually straightforward enough to compute, if the point at which the gradient is to be evaluated can be joined to each of the reference points by a unique length-minimizing geodesic and if one is willing to leave it in terms of the exponential mapping \cite{afsariRiemannian$L^p$Center2011} (in fact, a similar result holds in the classical \(L^2\)-Kantorovich geometry \cite{ambrosioGradientFlowsMetric2005}). One
has, for fixed \(\left\{ x_i \right\}_{i=1}^p \subset \M\), with \(\lambda \in \Delta^{p-1}\)
\begin{equation*}
    \nabla_x \left\{\sum\limits_{i=1}^p\frac{\lambda_i}{2}d^2(x, x_i)\right\} = -\sum\limits_{i=1}^p\lambda_i\exp_{x}^{-1}(x_i),
\end{equation*}
and therefore, what we will need to compute is 
\begin{equation*}
    \nu_{k+1} = \exp_{\nu_k}\left(-t_k\left(\sum\limits_{i=1}^{p}\lambda_i\exp_{\nu_k}^{-1}(\mu_i)\right)\right).
\end{equation*}

\section{Algorithms}
\subsection{Overview}
We now present our approach to realizing the barycentric coding model for measures on discrete metric spaces.  Synthesis is accomplished by minimizing the discrete variance functional \eqref{GraphVF} via intrinsic gradient descent in the \(\W\) metric. To be able to actually implement this minimization scheme we need to be able to do two things: evaluate the exponential map in order to push our iterates forward, and we also need to be able compute the intrinsic gradient of the functional at point. Analysis turns out to be significantly easier, because all we need to do in this case is form the Gram matrix defined by \eqref{GramMatrix} and minimize it with an off-the-shelf optimization library, such as CVX \cite{cvx, gb08}. Forming the Gram matrix requires approximating the inverse of the exponential map at the target measures \(\nu\) for each of the reference measures. This can be accomplished as a direct by-product of approximating the distance between the target and reference measures using the algorithm outlined in \cite{erbarComputationOptimalTransport2020}, as we will explain below.

\subsection{Computation of Optimal Transport on Discrete Spaces}
\label{DOTsec}
Having discussed the discretized transport metric on \(\mathcal{P}(\X)\), a natural question to ask is whether or not that distance can actually be computed in practice for a given graph and pair of measures. The answer to this is that while analytic computations of the distance are out of reach in general, at the least an approximation of the distance can be computed numerically. We recap here some essential definitions, constructions, and results from \cite{erbarComputationOptimalTransport2020}.

The insight in \cite{erbarComputationOptimalTransport2020} is to proceed via Galerkin discretization in time of the discrete continuity equation. By first constructing an appropriate discrete analog to the graph-transport metric, it is shown via a \(\Gamma\)-convergence technique that this discretization of the distance function converges to the true distance function in the limit of small step size. To that end, the following definitions are needed. Divide the unit interval \([0,1]\) into \(N\) subintervals \(I_i = [t_i, t_{i + 1})\) with uniform step size \(h = \frac{1}{N}, t_i = ih\) and define the following sets:

\[V_{n,h}^1 := \left\{ \psi_h \in C^0([0,1], \R^\X : \psi_h|_{I_i}) \text{ is affine } \forall i=0, \dots, N-1 \right\},\]
\[V_{n,h}^0 := \left\{ \psi_h :[0,1] \to \R^\X : \psi_h|_{I_i} \text{ is constant } \forall i=0, \dots, N-1 \right\},\]
\[V_{e,h}^0 := \left\{ \psi_h :[0,1] \to \R^{\X\times\X} : \psi_h|_{I_i} \text{ is constant } \forall i=0, \dots, N-1 \right\}.\]
We can think of these sets as representing, respectively: piecewise linear curves in the space of functions on \(\X\), their piecewise constant weak derivatives, and piecewise constant curves in the space of vector fields on \(\X\).
\begin{definition}
  The set of solutions to the \emph{discretized continuity equation} for given boundary values \(\rho_A, \rho_B \in \R^\X\) is given by
  \begin{multline*}
  \mathcal{CE}_h(\rho_A, \rho_B) = \Big\{ (\rho_h, m_h) \in V_{n,h}^{1}\times V_{e,h}^0 : \\
  h \sum\limits_{i=0}^{N-1}\left\langle \frac{\rho_h(t_{i+1}, \cdot) - \rho_{h}(t_i, \cdot)}{h} + \divx m_h(t_i, \cdot), \varphi_h(t_i, \cdot) \right\rangle_\pi = 0 \; \forall \varphi_h \in V_{n,h}^0 \Big\}.
  \end{multline*}
\end{definition}
\noindent  An example of such a weak solution is shown at a high level in Figure \ref{fig:ChambollePock}. The \emph{Galerkin discretization of the metric \(\W\)} is defined by the expression
 \[\W_h^2(\rho_A, \rho_B) = \inf \left\{\E_h(\rho, m) := \A(\rho, m) + \I_{\mathcal{CE}_h(\rho_A, \rho_B)}(\rho,m) \right\}.\]
The major analytic result of \cite{erbarComputationOptimalTransport2020} is reproduced below, and in words asserts that taking \(h \to 0\), one recovers the discrete transport metric \(\W\).

\begin{theorem}[\cite{erbarComputationOptimalTransport2020}, Theorem 3.6]
  Let \(\rho_A, \rho_B\) be fixed temporal boundary conditions. Then, the sequence of functionals \((\E_h)_h\) \(\Gamma\)-converges for \(h\to 0^+\) to the functional \(\E\) with respect to the weak* topology in \(\M([0,1], \R^\X \times \R^{\X\times\X})\).
\end{theorem}

An appropriate inner product can be defined so that the set \(H = V_{n,h}^1 \times (V_{e,h}^0)^4 \times (V_{n,h}^0)^2\) is a Hilbert space. Then, well-chosen slack variables are introduced so that the problem of computing \(\I_{\mathcal{CE}_h(\rho_A, \rho_B)}\) can be decomposed into several low-dimensional projection problems. A typical element of \(H\) looks like \(v = (\rho, m, \theta, \rho_{+}, \rho_{-}, \bar{\rho}, q) \); here \(\rho, m\) are in some sense the actual objects of interest; in brief, we have \(\rho\) and \(m\) representing the rectification of the geodesic curve and its momentum, while the additional components \(\theta, \rho_{+}, \rho_{-}, \bar{\rho},\) and \(q\) are the slack variables used for decomposing the projection problem. The components \(\theta, \rho_{+},\) and \(\rho_{-}\) are vector fields on the edge set, translating the mass on the nodes into mass averaged on the edges, while \(q\) and \(\bar\rho\) encode the average mass on the nodes over a time interval.  The Chambolle-Pock routine is carried out until the convergence threshold
\begin{equation*}
  \int_0^1\|\rho_{k+1} - \rho_{k}\|_{\pi}^2 \diff t < 10^{-10}.
\end{equation*}

The Chambolle-Pock algorithm, first introduced by its eponymous authors in \cite{chambolleFirstOrderPrimalDualAlgorithm2011}, is a primal-dual algorithm for solving non-smooth convex optimization problems, a form into which the discrete transportation geodesic approximation problem may be cast. Specifically, if \(F, G: H \to \R\) are convex functions on a Hilbert space, the Chambolle-Pock algorithm approximates the solution to

\begin{equation*}
  \min_{x \in H} F(x) + G(x).
\end{equation*}

It belongs to a broad class of extra-gradient methods \cite{korpelevichEXTRAGRADIENTMETHODFINDING1976}, such as the well-known Douglas-Rachford technique \cite{lionsSplittingAlgorithmsSum1979, condat2023proximal}. In order to apply the Chambolle-Pock routine to this problem, a set of slack variables is introduced to decouple certain non-linearities arising from the non-linear form of the action functional. Given convex functions \(F,G: H \to \R\), the Chambolle-Pock routine requires that we be able to compute the proximal mappings of the \(G\) and the Fenchel conjugate of \(F\) (denoted \(F^*\)). Recall that for an arbitrary function \(f\), its proximal mapping in \(H\) is defined by

\begin{equation*}
  \prox_f(x) := \argmin_{y \in H} \frac{1}{2}\|x - y\|_{H}^2 + f(y).
\end{equation*}
Note that if \(A \subset H\), and \(\mathcal{I}_A: H \to \R\) is defined so that \(\mathcal{I}_A(x) = 0\) for \(x \in A\) and \(+\infty\) otherwise, then 
\begin{equation*}
  \prox_{\mathcal{I}_A}(x) = \proj_A(x).
\end{equation*}
This observation allows for the algorithm given in \cite{erbarComputationOptimalTransport2020} to be expressed as the composition of many low-dimensional projection problems.

\begin{algorithm}
  \caption{Generic Chambolle-Pock Routine}
  \label{CP}
  \begin{algorithmic}
    \Require Convex function \(F, G\); \(\tau,\sigma > 0\) satisfying \(\tau\sigma < 1\), \(\lambda \in (0,1]\); initialization \(x^{(0)}, y^{(0)}, \bar{x}^{(0)}=x^{(0)}\)
    \Ensure \(x^*\) minimizing \(F(x) + G(x)\) 

    \While{not converged}
      \State \(y^{(k+1)} = \prox_{\sigma F^{*}}(y^{(k)} + \sigma \bar{x}^{(k)})\)
      \State \(x^{(k+1)} = \prox_{\tau G}(x^{(k)} - \tau \bar{y}^{(k+1)})\)
      \State \(\bar{x}^{(k+1)} = x^{(k+1)} + \lambda (x^{(k+1)} - {x}^{(k+1)})\)

    \EndWhile

    \State \Return \(x^{(k)}\)
	
  \end{algorithmic}
  
\end{algorithm}
\subsection{Approximating the Exponential Map with the Continuity Equation}
It is essential for our approach that this method of approximating the discretized \(\W\) metric proceeds by approximating geodesic curves connecting measures in \(\mathcal{P}(\X)\). Because the optimizers output by Algorithm \eqref{CP} include the momentum information associated to the geodesic curves, we are therefore able to approximate the intrinsic gradient of the discrete variance functional \eqref{GraphVF} at a point \(\nu\) by approximating the geodesics between \(\nu\) and each reference measure and summing the associated momenta, weighted according to the prescribed coordinates \(\lambda \in\Delta^{p-1}\).

In \cite{erbarComputationOptimalTransport2020}, an implicit Euler scheme in the spirit of the construction given in \cite{jordanVariationalFormulationFokkerPlanck1998} and expanded upon in \cite{ambrosioGradientFlowsMetric2005}, is constructed to compute gradient flows of the entropy for the discretized metric \(\W_h\).  It's classical formulation for a generic functional \(F\) is given by 
\begin{equation*}
   \rho_{k} := \argmin \left\{\frac{1}{2\tau}W_2^2(\rho, \rho_{k-1}) + F(\rho)\right\}.
\end{equation*}
This is sometimes referred to as a ``minimizing movements'' scheme, since the next iterate must minimize the sum of the objective functional and the \(W_2\) distance from the current iterate. In \cite{erbarComputationOptimalTransport2020}, this works because the entropy functional, defined by the equation
\begin{equation*}
  H(\rho) = \sum\limits_{x \in V}\rho(x)\log(\rho(x))\pi(x),
\end{equation*}
\begin{wrapfigure}{l}{0.4\textwidth}
  \centering
  \includegraphics[width=0.4\textwidth]{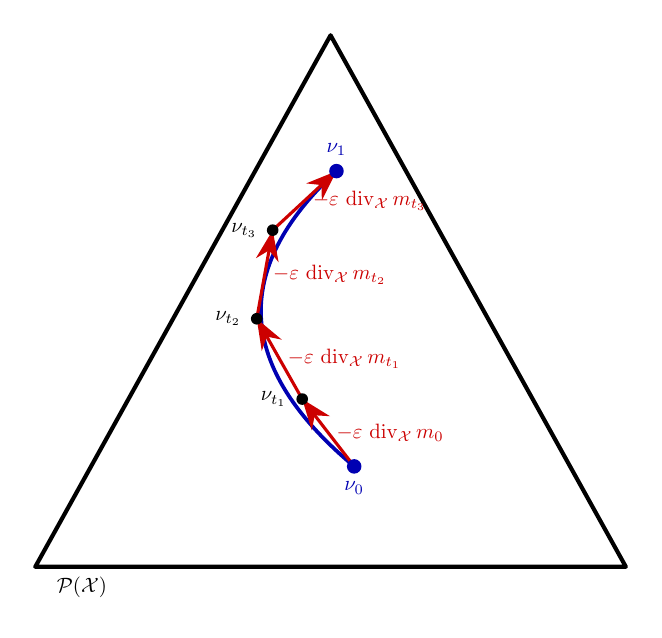}
  \caption{The Chambolle-Pock routine is used to approximate energy minimizing curves in \(\mathcal{P}(\X)\), ultimately producing a rectified curve encoded as a sequence of points and a collection of momentum vector fields. We leverage these vector fields to approximate the gradient of \eqref{GraphVF}. }
  \label{fig:ChambollePock}
\end{wrapfigure}
\noindent admits a Fenchel conjugate whose proximal mapping may be quickly approximated via the Moreau identity and a Newton scheme, and \(H^*\) can be folded neatly into the Chambolle-Pock objective function after certain modifications to the projection problems (see Section 6 of \cite{erbarComputationOptimalTransport2020} for full details). The implicit scheme is desirable from an analytic perspective because it is well-studied and is known to generalize the notion of a gradient flow to a generic metric space lacking linear structure --- however, it is not obvious how one would implement it numerically, which restricts its use in the practitioner's setting. In an ideal world, a similar approach might be used for computing the gradient flow associated to \eqref{GraphVF}, but computing the proximal mapping of \eqref{GraphVF} is itself as complicated as minimizing the functional itself to begin with. Therefore, the implicit Euler approach is not appropriate to the task, and so we instead proceed by an explicit approach.

First, we outline the general steps of the algorithm. In words, to synthesize a barycenter for a given dictionary with prescribed weights, we fix some initialization measure \(\nu_0\), and then for \(i=1, 2, \dots, \) we approximate the \(\W\) gradient at iterate \(\nu_i\) and update the iterate using a heuristic based on the continuity equation.  We stop iterating once the iterates converge in the \(\|\cdot\|_\pi\) norm, corresponding to a zero of the gradient vector field \(\nabla_{\W}\J\).
In order to carry out the update step of the intrinsic gradient descent scheme, it is necessary to have a way of ``stepping forward'' on the surface of the manifold. Generically the ideal situation would be to have some expression of the exponential map in terms of local coordinates, but this is rarely the case. Here we propose an approximation of the exponential mapping on \(\mathcal{P}(\X)\) by infinitesimal reasoning, leveraging the structure of the continuity equation. This is a formal argument at best: all geodesics in this manifold obey the continuity equation, but not all solutions to the continuity equation are geodesic curves. Nevertheless, our idea is that provided we do not move too far away from the initial point, we will move forward on the manifold in a way which is roughly distance minimizing. 

If \((\rho,m)\) is a pair which satisfies the continuity equation, then we have
\[\lim\limits_{\varepsilon\to 0}\frac{\rho(t + \varepsilon) - \rho(t)}{\varepsilon} = -\divx(m(t)).\]
Rearranging, taking \(\varepsilon > 0\), and evaluating at \(t = 0\), we then write
\begin{equation}
\rho(\varepsilon) \approx \rho(0) - \varepsilon\divx(m(0)).
\label{step}
\end{equation}
We will use \eqref{step} as our approximation of the exponential map. Applying this to \eqref{GraphVF}, we obtain our update step: 
\begin{equation*}
    \nu_{k+1} = \nu_k + t_k\divx\left(\sum\limits_{i=1}^p\lambda_im_{\nu,\nu_i}(0)\right).
\end{equation*}

\begin{algorithm}
  \caption{Synthesis of Barycenters via Intrinsic Gradient Descent in the \(\W\) Metric}
  \label{Alg1}
  \begin{algorithmic}
    \Require Markov chain \((Q,\pi)\), reference measures \(\left\{ \nu_i \right\}_{i=1}^p\), weights \(\lambda\in\Delta^{p-1}\), convergence criterion
    \Ensure Barycenter \(\nu = \Bary(\left\{ \nu_i \right\}_{i=}^p), \lambda)\)

    \State \(\nu \gets \nu_1\) \Comment{Initialize with a measure which is trivially in \(\Bary(\left\{ \nu_i \right\}_{i=1}^p)\)}

    \While{not converged}
      \State \(m \gets 0 \in \R^{|V|\times|V|}\)
      \For{\(r=1,\dots,p\)} 
        \State \(m \gets m + \lambda_r m_{\nu,\nu_r}(0)\) \Comment{Compute the initial momentum of the geodesic \([\nu, \nu_r]\) and sum}
        \EndFor
      \State \(\nu \gets \nu + \varepsilon\divx m \) \Comment{Update via continuity equation heuristic}
    \EndWhile

    \State \Return \(\nu\)
	
  \end{algorithmic}
  
\end{algorithm}

A natural concern here is whether or not \eqref{step} remains in the admissible set \(\mathcal{P}(\X)\). We show that this is almost the case with the following lemma.
\begin{lemma}
\label{lemma1}
Let \(\nabla\varphi \in T_{\nu}\mathcal{P}(\X)\). Then 
\begin{equation*}
\sum\limits_{x \in V}\divx \nabla\varphi(x)\pi(x) = 0.
\end{equation*}
\end{lemma}
\begin{proof}
    The claim follows from the reversibility of the pair \((Q, \pi)\). By definition of the graph gradient, \(\nabla\varphi\) is a skew-symmetric matrix. Therefore, we have 
    \begin{align*}
        \sum\limits_{x\in V}\divx\nabla\varphi(x)\pi(x) &= \sum\limits_{x\in V}\frac{1}{2}\sum\limits_{y \in V}Q(x,y)(\nabla\varphi(y,x) - \nabla\varphi(x,y))\pi(x), \\
      &=\sum\limits_{x,y\in V}\pi(x)Q(x,y)\nabla\varphi(y,x). 
    \end{align*}
Therefore we may interpret the graph divergence of \(\nabla\varphi\) as an elementwise sum of a certain matrix \(\hat{Q}\odot \nabla\varphi\), where \(\hat{Q}\) is the matrix \(\hat{Q} = [\pi(x)Q(x,y)]_{x,y \in \X}\) and \(\odot\) is the Hadamard product for matrices (i.e., elementwise multiplication). By reversibility of the pair \((Q, \pi)\), the matrix \(\hat{Q}\) is symmetric, so its Hadamard product with \(\nabla\varphi\) is skew-symmetric. Observing that the elementwise sum of any skew-symmetric matrix is \(0\) finishes the proof. 
\end{proof}

\begin{corollary}
Let \(\nu \in \mathcal{P}(\X)\). Then \(\sum\limits_{x\in V}(\nu(x) + \varepsilon\divx\nabla\varphi(x))\pi(x)=1\).
\end{corollary}

In other words, while it could in principle be possible to leave the probability simplex by violating the non-negativity constraint, the updates will remain on the hyperplane \(\{\langle\rho,1\rangle = 1\}\).

\subsection{Approximation of the Exponential Map's Inverse via Chambolle-Pock}
Recalling that the exponential mapping of tangent vector \(v \in T_p\M\) is defined to be \(\gamma_{p,v}(1)\), where \(\gamma_{p,v}\) is a geodesic curve starting at \(p\) with initial velocity \(v\) and time domain \([0,1]\), it follows then that if we had an oracle which could produce for arbitrary \(p,q\in\M\) a geodesic \(\gamma\), then the inverse of the exponential mapping at \(p\) could be computed by evaluating \(\frac{d}{dt}\gamma(t)|_{t=0}\). Fortunately this is precisely one of the outputs of Algorithm \eqref{CP}. Forming the linear system associated to the quadratic program in \eqref{AP} is therefore as simple as approximating the geodesic between the target measures \(\nu\) and the dictionary of reference measures \(\left\{ \nu_i \right\}\), computing the pairwise inner products of their initial momentum vectors in the space \(T_\nu \mathcal{P}(\X)\) and invoking an off-the-shelf quadratic programming solver.

\begin{algorithm}
  \caption{Analysis of Barycenters via Quadratic Programming}
  \label{Alg2}
  \begin{algorithmic}
    \Require Markov chain \((Q,\pi)\), target measure \(\nu\), reference measures \(\left\{ \nu_i \right\}_{i=1}^p\)
    \Ensure Coordinates \(\lambda^*\) such that \(\nu = \Bary(\left\{ \nu_i \right\}_{i=}^p, \lambda^*)\)

    \State \(A \gets 0 \in \R^{p\times p}\)
    \For{\(i=1,\dots,p\)}
      \State \(\nabla\varphi_i \gets m_{\nu,\nu_i}(0)\) \Comment{Compute the initial momentum of the geodesic \([\nu,\nu_r]\)}
    \EndFor
    \For{\(i,j=1,\dots,p\)}
      \State \(A_{i,j} \gets \frac{1}{2}\sum\limits_{x,y\in \X}\theta(\nu(x), \nu(y))\nabla\varphi_i(x,y)\nabla\varphi_j(x,y)\) \Comment{Form the Gram matrix}
    \EndFor

    \State \(\lambda^*\gets \min\left\{ \lambda^TA\lambda : \lambda \in \Delta^{p-1} \right\}\) \Comment{Solve via CVX, etc.}

    \State \Return \(\lambda^*\)
	
  \end{algorithmic}
  
\end{algorithm}

\subsection{Hyperparameters for the Synthesis and Analysis Algorithms}
Several parameters need to be set when running these algorithms which are divorced from the relevant measures and weights themselves; Table \ref{table:Hyperparameters} summarizes which hyperparameters are involved in which subroutines. First, there are two hyperparameters associated to the underlying Chambolle-Pock routine computation of action minimizing geodesics: the convergence tolerance, \(\delta_g\) and the number of time steps in each geodesic, \(N\). For the synthesis computation, there is also the descent step size \(\varepsilon\), and the convergence threshold for the descent scheme \(\delta_b\).

The parameters \(\delta_g\) and \(N\) control the quality of the approximated geodesics in the \(\W\) geometry. As \(\delta_g \to 0\) and \(N \to \infty\), the outputs of the Chambolle-Pock routine become finer and finer, but at the cost of computational complexity --- note that the subroutines (proximal map computations) have time complexity growing quadratically in \(N\) (see \cite{erbarComputationOptimalTransport2020} for full details), and that in general Chambolle-Pock exhibits a rate of convergence \(\mathcal{O}(1/k)\) in the number of iterations. 

The descent parameter \(\delta_b\) governs the threshold for convergence of the synthesis scheme, while \(\varepsilon\) controls the size of the step taken after each gradient computation. In principle, intrinsic gradient descent, under certain assumptions on the underlying manifold, exhibits \(\mathcal{O}(1/k)\) rate of convergence, and an optimal descent step size can be chosen for the constant step size paradigm as detailed in \cite{afsariConvergenceGradientDescent2013}, but such a choice requires knowledge of the injectivity radius of \((\Delta^{p-1}, \W)\) and an upper bound on the eigenvalues of the Riemannian Hessian of the variance functional \(D^2\J\). At present, these quantities are unstudied, and we leave investigation of them to future work; here we simply remark that we find that taking \(\varepsilon \in (0.1, 0.5)\) generally results in reasonable rates of convergence. In recent years, interest has been taken in accelerating the rate of convergence for intrinsic gradient descent \cite{alimisisMomentumImprovesOptimization2021}, taking inspiration from techniques well-known in Euclidean optimization \cite{polyak1987introduction}, and this may also be a useful direction in which to expand the present work.

\begin{table}
  \centering
  \begin{tabular}{| c | c | c | c | c |} 
 \hline
 Routine        & \(\delta_g\)    & \(N\)      & \(\delta_b\)    & \(\varepsilon\)\\ [0.5ex] 
 \hline\hline

 Chambolle-Pock & \checkmark & \checkmark &            &  \\ [1ex] 
 \hline

 Synthesis      & \checkmark & \checkmark & \checkmark & \checkmark  \\ [1ex] 
 \hline
 Analysis       & \checkmark & \checkmark &            &             \\ [1ex] 
 \hline
\end{tabular}
  \caption{Table of hyperparameters associated to each algorithm}
  \label{table:Hyperparameters}
\end{table}

\subsection{Barycentric Coding Model via Entropic Regularization}
In \cite{bonneelWassersteinBarycentricCoordinates2016}, an algorithm is proposed for performing transport-based histogram regression for measures on fixed support. For a given family of histograms \(\left\{ p_s \right\}_{s=1}^S\), consider the mapping 
\begin{equation*}
  P: \lambda \mapsto \argmin \sum\limits_s \lambda_sOT_{\varepsilon_{reg}}(p,p_s),
\end{equation*}
and seek solutions to the optimization problem
\begin{equation*}
  \argmin \E(\lambda) := \mathcal{L}(P(\lambda), q),
\end{equation*}
for prescribed reference histogram \(q\), where \(\mathcal{L}\) is some common and easily differentiated loss function (e.g., the squared \(L^2\) difference, or the Kullback-Liebler divergence). While it is expensive to compute the Jacobian \([\partial P(\lambda)]^T\) needed for the naive descent scheme, they show that replacing it by a running estimate \([\partial^{(L)} P(\lambda)]^T\) obtained after \(L\) Sinkhorn iterations makes the problem computationally tractable.

In \cite{gentile2025regularized}, that approach was leveraged to implement an entropically regularized barycentric coding model for measures supported on graphs, using diffusion distances as ground metrics encoding the connectivity information of the graph. There are trade-offs associated to this approach to barycenters of measures on graphs. Firstly, there is the question of how to choose the cost matrix \(C\).  As alluded to previously, there is not an isomorphism between pairwise distance matrices and graph structures --- it is perfectly possible to have two identical cost matrices which represent different network structures. Secondly, the Bregman projections algorithm, which depends on the entropic regularization, both in regularizing parameter \(\varepsilon_{reg}\) and choice of reference measure (typically taken to be the tensor product of the reference measures) is sensitive to the choice of regularization parameter \(\varepsilon_{reg}\), with instability of the computations occurring in the limit as \(\varepsilon_{reg} \to 0\).  Strictly speaking, one could choose to forgo the regularization of the transport cost and try solving the linear program directly, but this is not without issues of its own. One such issue is that solving linear programs is expensive, but another issue is that even if we are happy to pay the cost of solving an expensive linear program, a fundamental problem remains in that the geometry induced by the Kantorovich problem does not admit non-trivial geodesics \cite{maasGradientFlowsEntropy2011}.

In practice this trade-off may be worth it, since synthesizing barycenters via the Sinkhorn algorithm is much faster than the gradient descent approach outlined above.  However, it is worth emphasizing that the introduction of the entropic regularization term introduces bias and in fact computation of entropy-regularized transport cost can be interpreted as performing a maximum-likelihood estimation for Gaussian deconvolution \cite{rigolletEntropicOptimalTransport2018} --- in short, it causes a loss in geometric information which we see realized as a kind of ``blurring'' of the support of the measures across the entire domain.

Summarizing then, we are comparing two approaches which could be taken to the problem of synthesizing and analyzing barycenters of measures on graphs: one is the ``static'' approach born of the Kantorovich problem, and the other is the ``dynamic'' approach born of the transport metric proposed in \cite{maasGradientFlowsEntropy2011}. The former offers the following trade-off: if we are willing to make a non-canonical choice about how to encode the connectivity of a given graph and choose a workable loss function, and as well introduce the entropic regularization, then a combination of Bregman projections, the Sinkhorn-Knopp algorithm, and gradient descent can achieve system for synthesizing and analyzing barycenters of measures on graphs. The latter requires us on the other hand to set certain tolerances for convergence, but otherwise allows us to construct a barycentric coding model using a canonical representation of the graph via its Markov chain random walk transition matrix.

\section{Experiments}
Here we present some experimental results based on our numerical implementation of the outlined scheme above and highlight our major observations.  Broadly speaking, we find that our intrinsic descent scheme is able to synthesize and analyze barycenters of measures on graphs. In Figures \ref{fig:Hypercube} and \ref{fig:Massachusetts} we demonstrate the general capability of our program. Figure \ref{fig:Hypercube} features the barycenter of \(4\) measures on a graph with 16 nodes, and for both of the weighted and unweighted edge cases. Figure \ref{fig:Massachusetts} features two barycenters of measures defined on a spatial geography graph where each vertex corresponds a district for the Massachusetts state house, with edges indicating that two districts share a border (in total, there are \(156\) vertices in the graph and \(404\) edges); the two barycenters share the same references and prescribed coordinates, but differ in terms of their overall convergence threshold \(\delta_b\).

\begin{figure}
    \centering
    \includegraphics[width=0.95\linewidth]{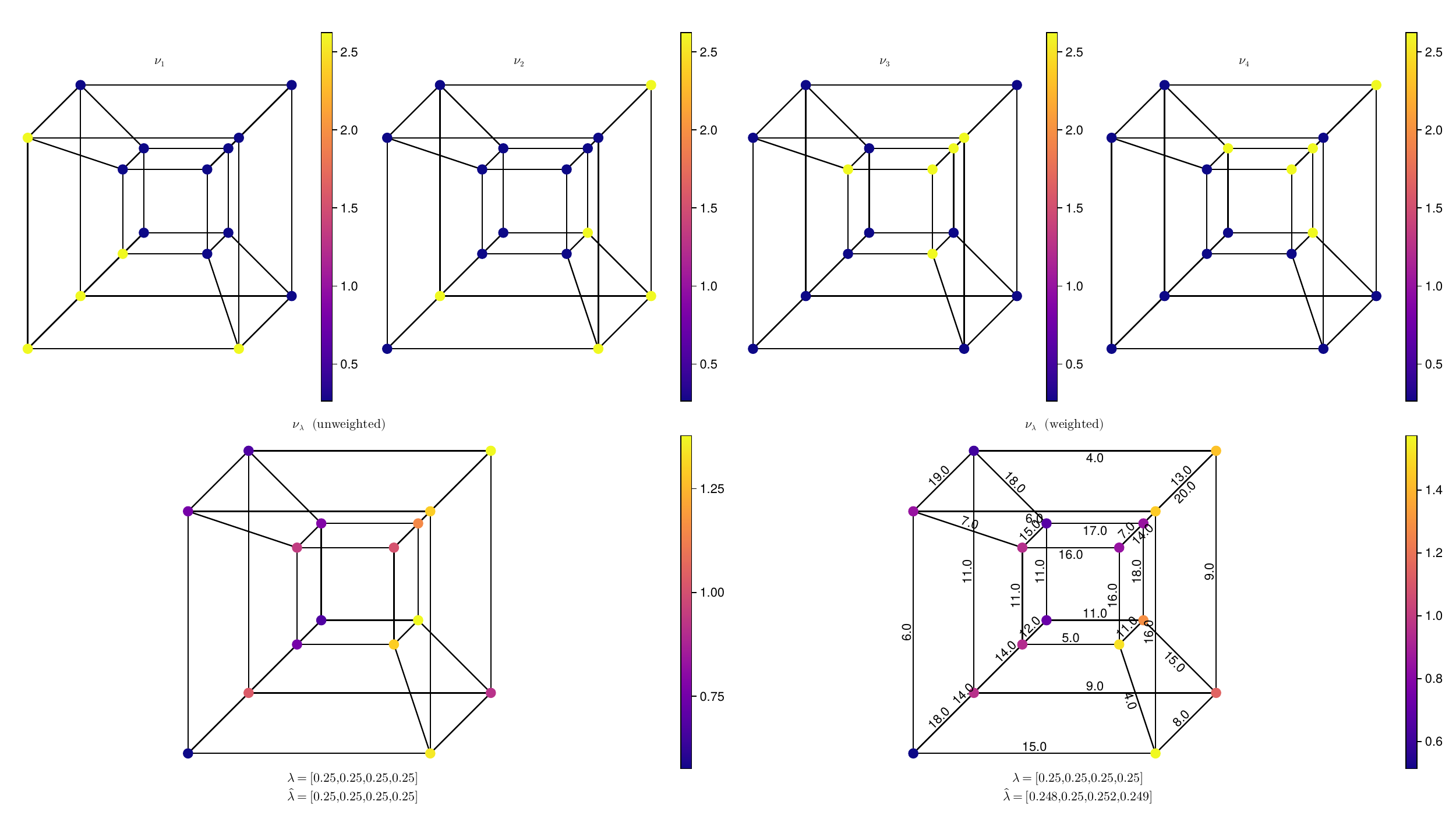}
    \caption{Barycenter of 4 measures on the graph formed by the edges and vertices of a hypercube. Bottom left: edges of the hypercube are uniformly weighted with \(\omega_{i,j}\equiv 1\); bottom right: edges are randomly weighted (weights are symmetric to preserve the reversibility condition of the Markov chain). Parameters used to compute barycenters: \(\varepsilon = 0.1, N=10, \delta_g=10^{-10}, \delta_b=10^{-10}\).}
    \label{fig:Hypercube}
\end{figure}

\begin{figure}
    \centering
    \includegraphics[width=0.7\linewidth]{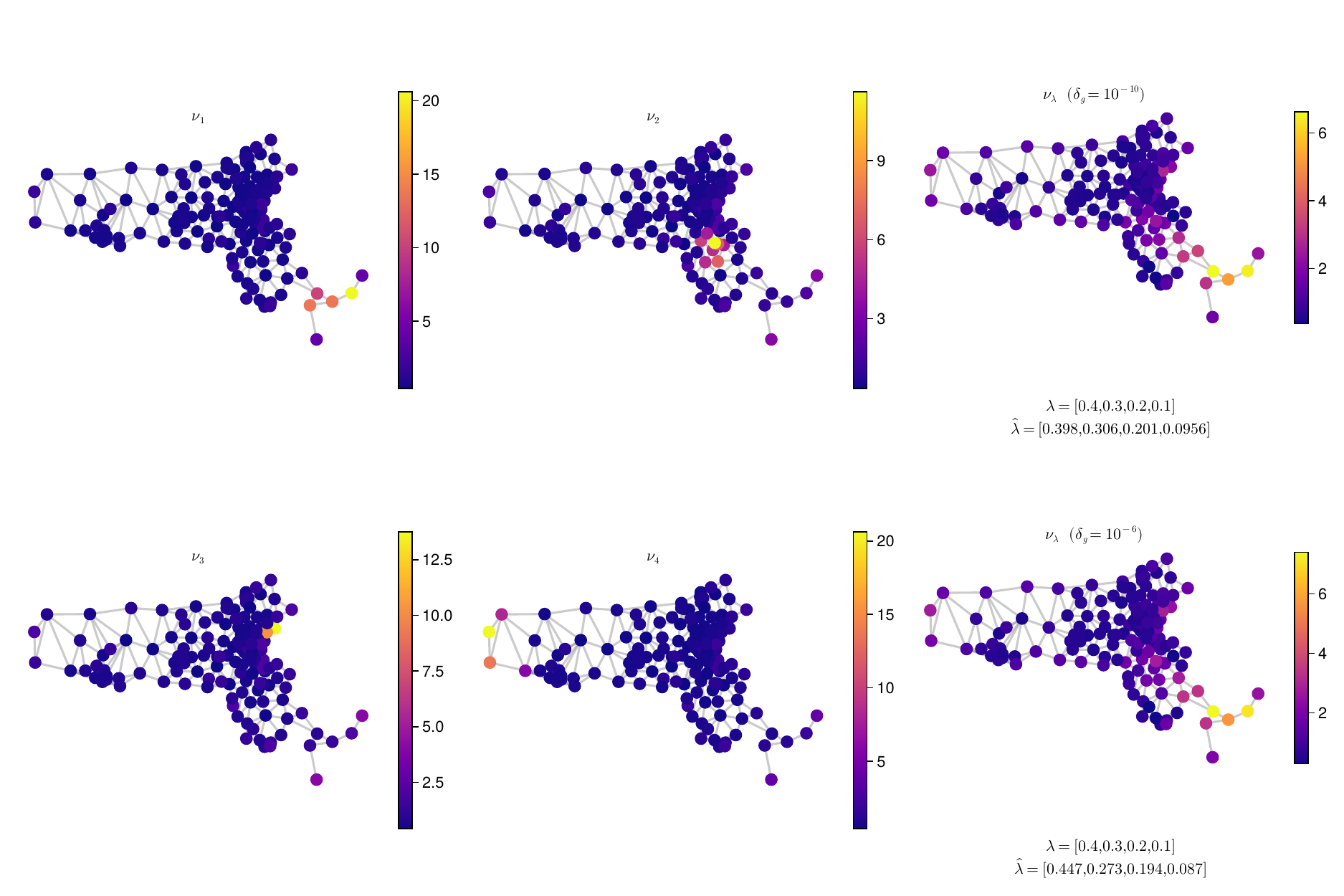}
    \caption{Barycenter of 4 measures on the graph formed by the spatial geography node constructed from the Massachusetts state house districts. Left: 4 randomly generated geographically concentrated reference measures; right: barycenters with weights \(\lambda = [0.4, 0.3, 0.2, 0.1]\) and recovered coordinates; top right, barycenter computed with \(\delta_g = 10^{-10}\); bottom right: barycenter computed with \(\delta_g=10^{-6}\). We observe that more rigorous convergence tolerance in the Chambolle-Pock routine leads to superior coordinate recovery in the analysis problem.  Common parameters used to compute barycenters: \(\varepsilon = 0.1, N=2, \delta_b=10^{-8}\). Data retrieved from \url{https://www.mass.gov/info-details/massgis-data-massachusetts-house-legislative-districts-2021}.}
    \label{fig:Massachusetts}
\end{figure}

\subsection{Comparison: Recovering Geodesic Curves via Wasserstein Gradient Descent}

The first test we can conduct to verify our algorithm's outputs is to consider the case of \(2\) reference measures. As pointed out in \cite{aguehBarycentersWassersteinSpace2010}, if we fix reference measures \(\nu_0, \nu_1\) and denote by \(\nu(t) : [0,1] \to \mathcal{P}(\X)\) the geodesic connecting \(\nu_0\) to \(\nu_1\), with \(\nu(0) = \nu_0\), then a consequence of the triangle inequality is that
\begin{equation*}
  \nu(t) = \argmin\limits_{\nu} \left\{(1 - t) \W^2(\nu_0, \nu) + t \W^2(\nu_1, \nu)\right\}.
\end{equation*}
This implies that computing a barycenter of two measures requires no additional bells and whistles beyond a mechanism for computing the geodesics between them. 

\begin{figure}[htbp]
  \begin{subfigure}[t]{0.45\textwidth}
  \includegraphics[width=\linewidth]{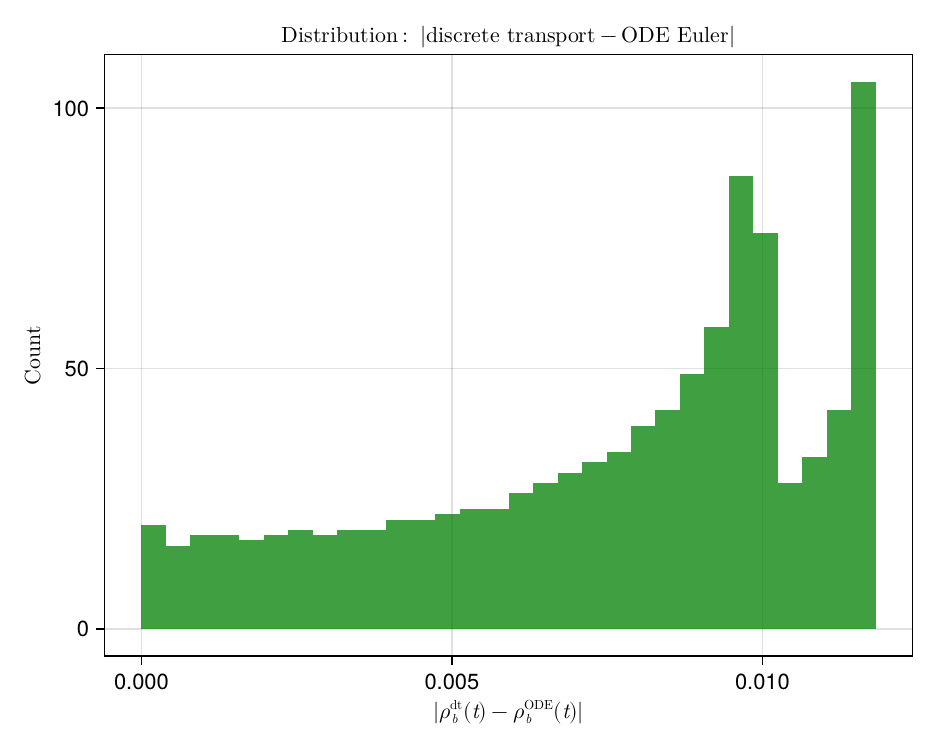}
  \caption{Distribution of deviation between Chambolle-Pock and Euler methods for approximating mass on node \(b\) across geodesic between Dirac masses on a two point graph.}
  \end{subfigure}
 \hfill 
\begin{subfigure}[t]{0.45\textwidth}
  \includegraphics[width=\linewidth]{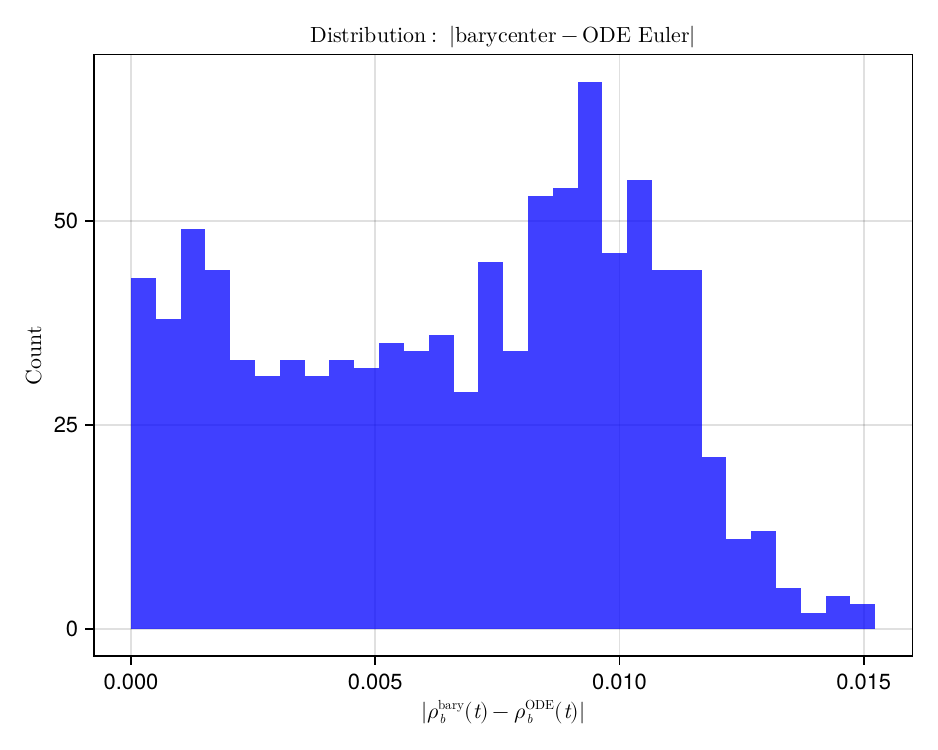}
   \caption{Distribution of deviation between intrinsic descent and Euler methods for approximating mass on node \(b\) across geodesic between Dirac masses on a two point graph.}
\end{subfigure}
\caption{Comparison of barycenters on a \(2\) node graph computed by solving an ODE vs computed by intrinsic gradient descent.}
\label{fig:ODEexperiment}
\end{figure}

We will proceed in two distinct ways for this procedure. First, we note that in the case of a two-point space \(\X_2 = \{a,b\}\), there exists a simple ODE governing the geodesic curve between any given pair of measures \(\nu_0, \nu_1 \in \mathcal{P}(\X_2)\). Here, we take \(\nu_0 = \delta_a(x)\) and \(\nu_1 = \delta_b(x)\) to be Dirac masses concentrated on either of the nodes in the graph. Then, we can solve the ODE via an explicit Euler method, and compare the geodesic predicted by the solution of the ODE to the barycenters produced by our descent method.  All three methods produce similar geodesics, and we compare them directly by looking at the amount of mass on node \(b\), taking geodesic obtained by the Euler method as our ground truth; see Figure \ref{fig:ODEexperiment} for the distribution of results.

For the other approach to this experiment, we can on the one hand compute the geodesic \(\nu(t)\) between \(\nu_0, \nu_1\) using the Chambolle-Pock algorithm, evaluate that geodesic at different values of \(t = t_0, t_1, \dots\), and see how those evaluations compare to the barycenters obtained by, on the other hand, applying our descent method to the reference measures \(R = \{\nu_0, \nu_1\}\) with weights \(\lambda = (1 - t_0, t_0),  \dots, (1 - t_N, t_N)\). We perform two versions of this approach to the experiment. First, we fix a single pair of measures \(\left\{ \nu_0, \nu_1 \right\}\) on a graph \(G\), compute the points on the geodesic \(\left\{ \nu_{t_i} \right\}_{i=0}^{N}\) between them with \(N=1000\) time steps, and then for each \(i=1, \dots, N - 1\), compute the barycenter \(\argmin\left\{ (1 - i/N)\W^2(\nu,\nu_0) + (i/N)\W^2(\nu,\nu_1) \right\}\). We treat the measures \(\nu_{t_i}\) output by the Chambolle-Pock computation of the geodesic curve between \(\nu_0\) and \(\nu_1\) as the ground truth and report the relative error of the barycenters computed via intrinsic gradient descent; the results of this experiment are plotted in Figures \ref{fig:NormDiffsDistribution} and \ref{fig:NormDiffsInTime}, both as an ensemble distribution and as a function of the interpolation parameter \(t\). Then we perform a variation on the experiment, keeping \(G\) fixed but now letting the measures and coordinates be randomly generated each time, see Figure \ref{fig:NormDiffsRandomized}. In order to generate the random measures on the triangle, we sample from the uniform distribution on the interval \([0,1]\) three times, and normalize by dividing through by the sum of values, and then componentwise by the steady state vector \(\pi\).

We also hasten to point out that taking the measures along the computed geodesic as ground truth it not necessarily optimal. Note that the object being optimized by the Chambolle-Pock routine is a representation of the entire geodesic. While in the limit of \(N \to \infty\) we would expect perfect correspondence between points on the geodesic and barycenters of \(\nu_0\) and \(\nu_1\), in the case of finite step size we must anticipate some difference between the true barycenters and the measures along the approximated curve. Nevertheless, the small relative error in our experiments indicates that the intrinsic descent scheme does converge towards a variance minimizing measure.

\begin{figure}[htbp]
  \centering
  \begin{subfigure}[t]{0.48\textwidth}
  \includegraphics[width=\linewidth]{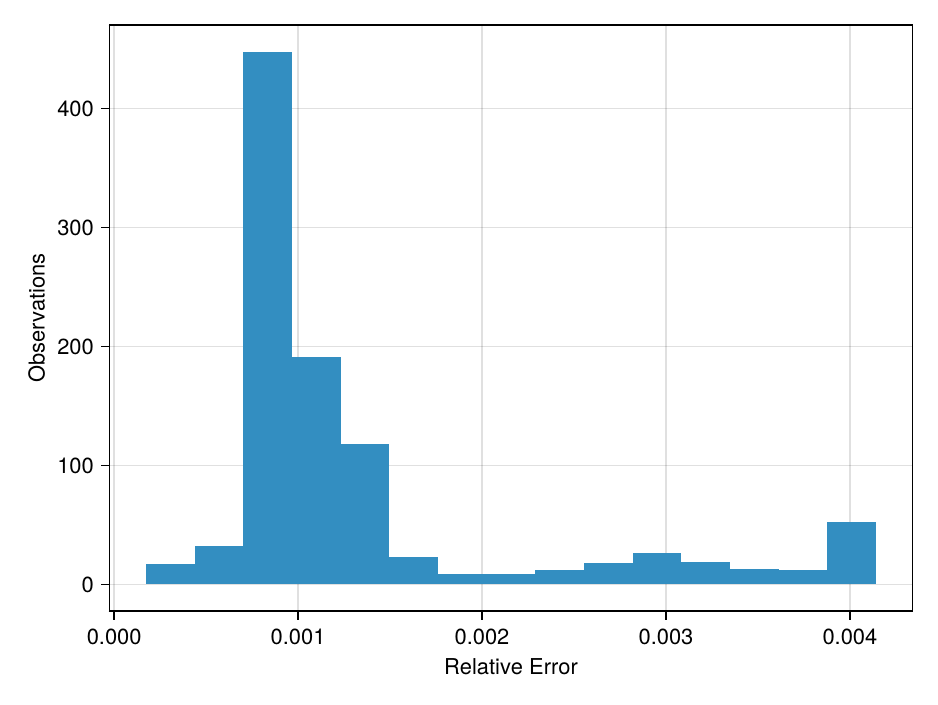}
  \caption{Distribution of 1000 computed relative norm differences between (a) the barycenter of two fixed measures \(\nu_0, \nu_1\) on a triangle graph and (b) the corresponding point on the geodesic between \(\nu_0\) and \(\nu_1\).}
  \label{fig:NormDiffsDistribution}
  \end{subfigure}
 \hfill 
\begin{subfigure}[t]{0.48\textwidth}
\includegraphics[width=\linewidth]{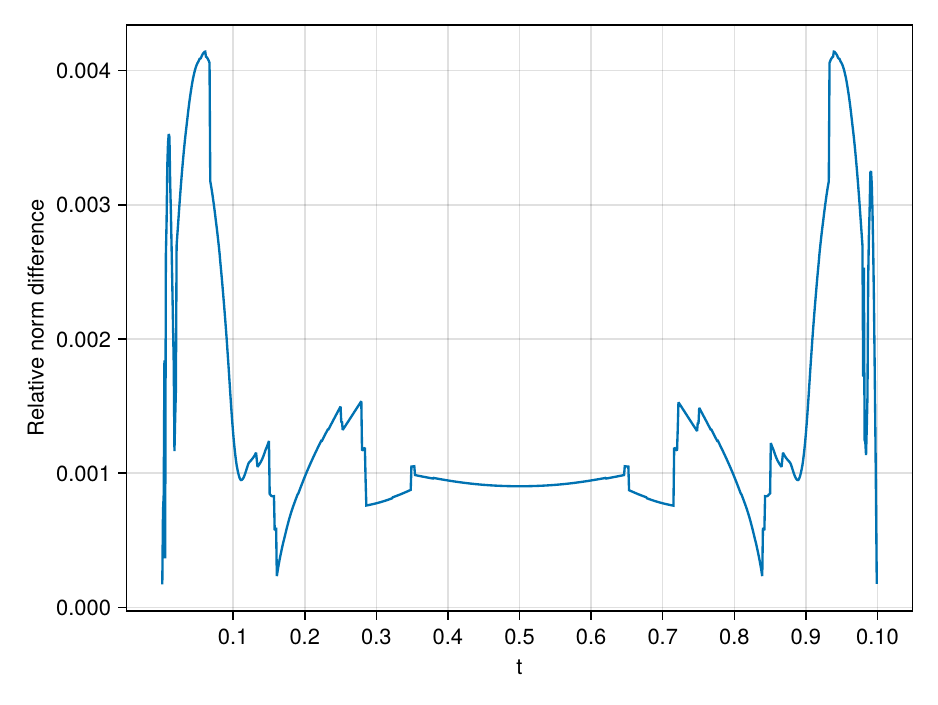}
  \caption{1000 computed relative norm differences between (a) the barycenter of two fixed measures \(\nu_0, \nu_1\) on a triangle graph and (b) the corresponding point on the geodesic between \(\nu_0\) and \(\nu_1\) as a function of the interpolation parameter \(t\).}
  \label{fig:NormDiffsInTime} 
\end{subfigure}
\caption{Comparison between barycenters of two measures computed as geodesics via Chambolle-Pock and via intrinsic gradient descent}
\end{figure}

\begin{figure}[htbp]
  \centering
  \includegraphics[width=0.7\linewidth]{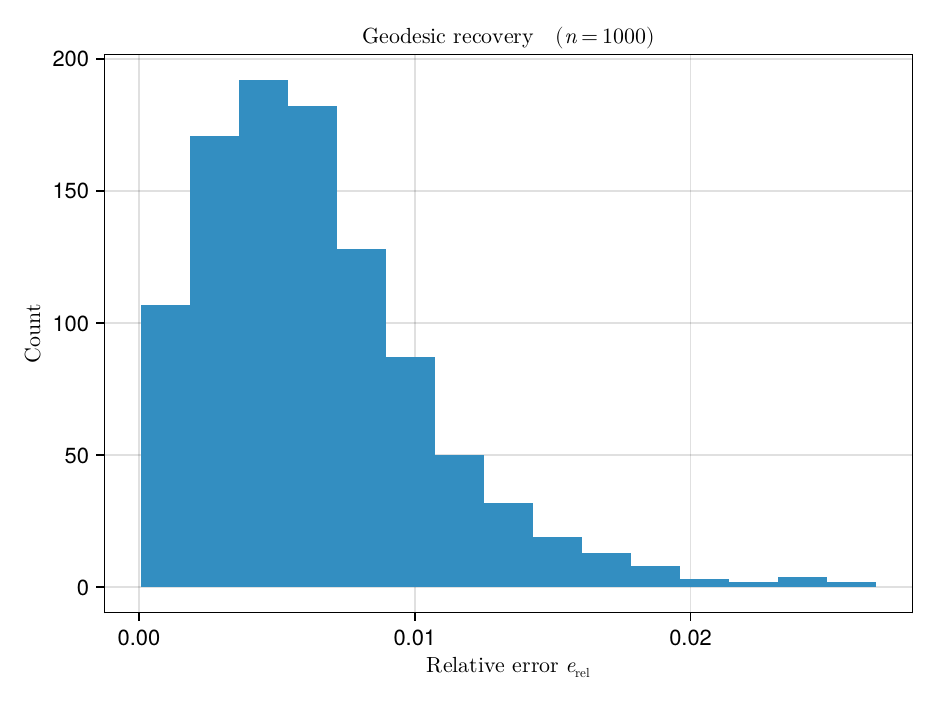}
  \caption{Distribution of 1000 computed relative norm differences between (a) the barycenter \(\nu_b\) of two randomly generated measures \(\nu_0^i, \nu_1^i\) on a triangle graph and (b) the corresponding point on the geodesic between \(\nu_0^i\) and \(\nu_1^i\), \(\nu_g\). When computing the relative error we assume that the point on the geodesic is a genuine representative of the barycenter, and therefore compute \(\|\nu_{g} - \nu_{b}\|_{\pi} / \|\nu_g\|_\pi\)}
  \label{fig:NormDiffsRandomized}
\end{figure}

\subsection{Comparison: Barycenters Obtained by Gradient Descent vs by Entropic Regularization}
Here we compare qualitatively the barycenters obtained by our dynamic approach to those obtained via the entropically regularized static approach. We will compare our algorithm's outputs to those obtained when the cost matrix used in formulating the entropic regularization problem is (1) the squared shortest path metric on the graph and (2) the squared diffusion distance between graph vertices for various values \(t > 0\).

For our first comparison between the methods, we fix a grid graph on \(49\) vertices, generate \(3\) random measures and coordinates \(\lambda = (0.5, 0.3, 0.2)\), and then compute the barycenter \(\nu_\lambda\) \((1)\) by intrinsic gradient descent, \((2)\) by entropic regularization with the ground metric given by the squared shortest path metric and \((3)\) by entropic regularization with the squared diffusion distance at time \(t=12\), shown in Figure \ref{fig:GridEntropicComparison}. For our second comparison, we repeat the process outlined above on a graph representing the 48 contiguous U.S. States + D.C, this time taking \(t=11\) for the diffusion distance, shown in Figure \ref{fig:StatesEntropicComparison}. Diffusion times in these experiments are selected by setting \(t\) according to the diameter of the graph, i.e., the maximal path length among all shortest paths between pairs of vertices.

In order to generate the measures which have some semblance of geographic concentration we perform the following procedure. First, we pick an arbitrary node \(v \in V\), and let \(N(v)\) be the set of vertices connected to \(v\), together with \(v\). Then we generate a weighting vector, for a set weight parameter \(w > 1\), and assign mass \(1\) to \(u \notin N(v)\) and mass \(w\) to \(u \in N(v)\). Finally we normalize the resulting vector so that it lies in \(\mathcal{P}(\X)\), which is done by dividing componentwise by the steady state vector \(\pi\) and the total unnormalized mass \(|N(v)^c| + w|N(v)|\). The parameter \(w\) roughly controls how close the measures are to the boundary of \(\mathcal{P}(\X)\). In these specific experiments, we take \(w=100\).

Some commonalities jump out from these experiments. For brevity, let us refer to the barycenter computed with Algorithm \ref{Alg1} as \(\nu_{dyn}\), the one computed via entropic regularization with cost given by the squared shortest path metric as \(\nu_{st,sp}\), and the one computed via entropic regularization with cost given by the squared diffusion distance as \(\nu_{st, diff}\). 
First, we look at the results for the grid graph. Qualitatively, all three barycenters have supported concentrated in the lower left corner of the graph, but none of them are actually similar to each other. \(\nu_{dyn}\) has support which extends along the off-diagonal of the grid, while \(\nu_{st,sp}\) has support mostly concentrated below the main diagonal and parallel to it, and \(\nu_{st,diff}\) is significantly more concentrated than either of the others. In terms of recovery error, \(\nu_{st,sp}\) dominates with a relative error of approximately \(2.43\times 10^{-7}\), while \(\nu_{dyn}\) and \(\nu_{st,diff}\) both admit relative recovery errors of approximately \(0.006\). 

Turning our attention to the U.S. states graph, we see that \(\nu_{st,diff}\) significantly improves its relative recovery error to roughly match the relative error obtained for \(\nu_{st,sp}\), the former coming in at \(6.02\times10^{-8}\) and the latter at \(5.24\times10^{-8}\). Recovery is not much changed for \(\nu_{dyn}\) on the graph, still at \(0.006\). Also of note in regards to the U.S. states graph, the barycenters computed via the static approaches are much more visually similar to each other than they are in the case of the grid. 

These experiments highlight the trade-offs between approaching the barycentric coding model from either the static or dynamic point of view. On the one hand, results can be obtained quickly using the static formulation with entropic regularization, but at the cost of having to make a good choice about what ground metric to use, while on the other hand, consistent results can be obtained without having to make such a choice, at the cost of significantly more compute time when using the dynamic formulation.

\begin{figure}
  \centering
    \includegraphics[width=0.84\linewidth]{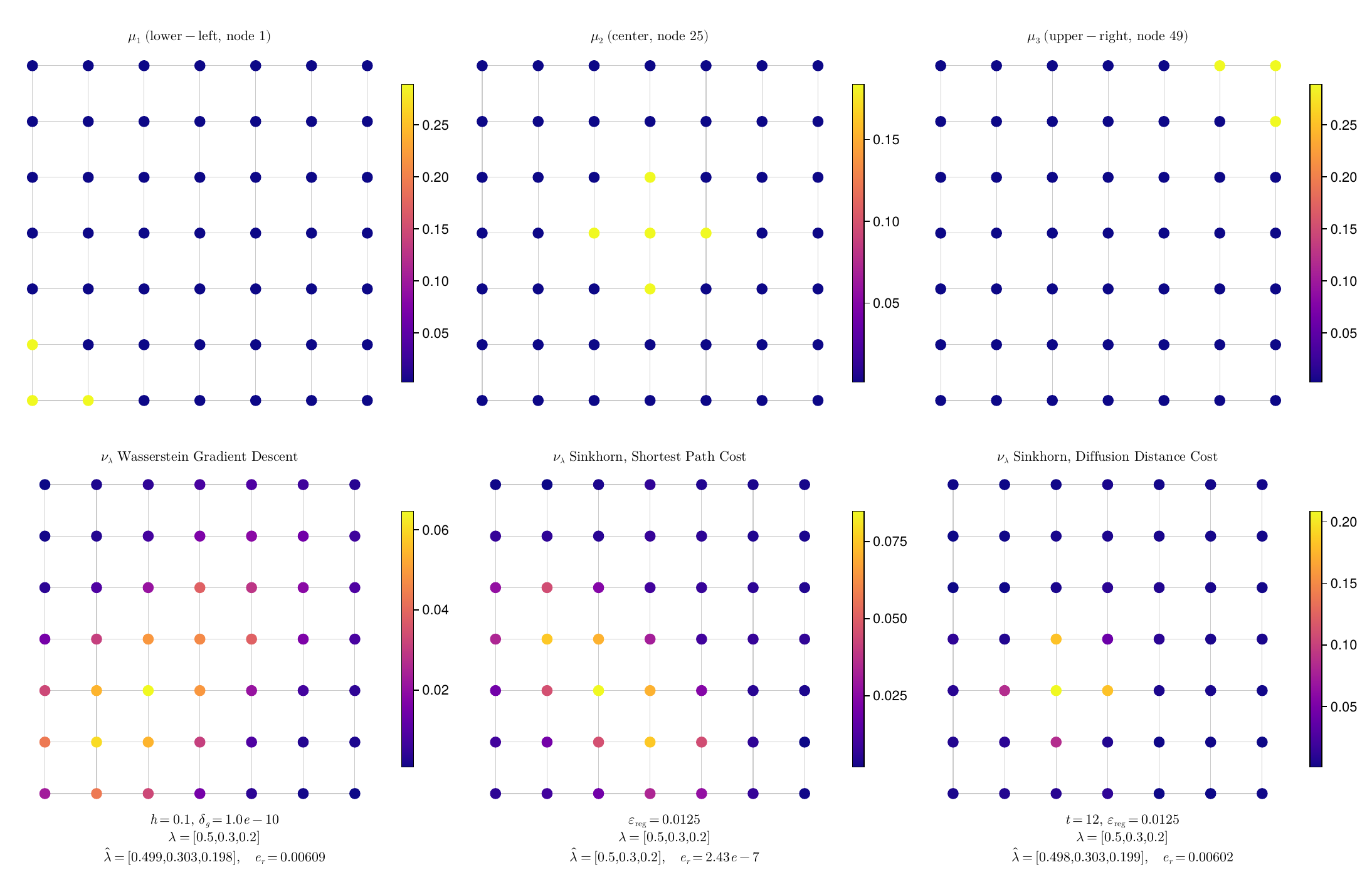}
  \caption{Barycenter of three locally concentrated measures on a grid graph on \(49\) vertices. Bottom left: computed via intrinsic gradient descent and the dynamic formulation of discrete transport; bottom center: computed via entropic regularization with cost given by the squared shortest path metric; bottom right: computed via entropically regularized transport, with the cost matrix given by the squared diffusion distance for \(t=12\). Parameters used to compute barycenters: \(\varepsilon = 0.1, N=3, \delta_g = 10^{-12}, \delta_b = 10^{-10}\).}
  \label{fig:GridEntropicComparison}
\end{figure}

\begin{figure}
  \centering
    \includegraphics[width=0.84\linewidth]{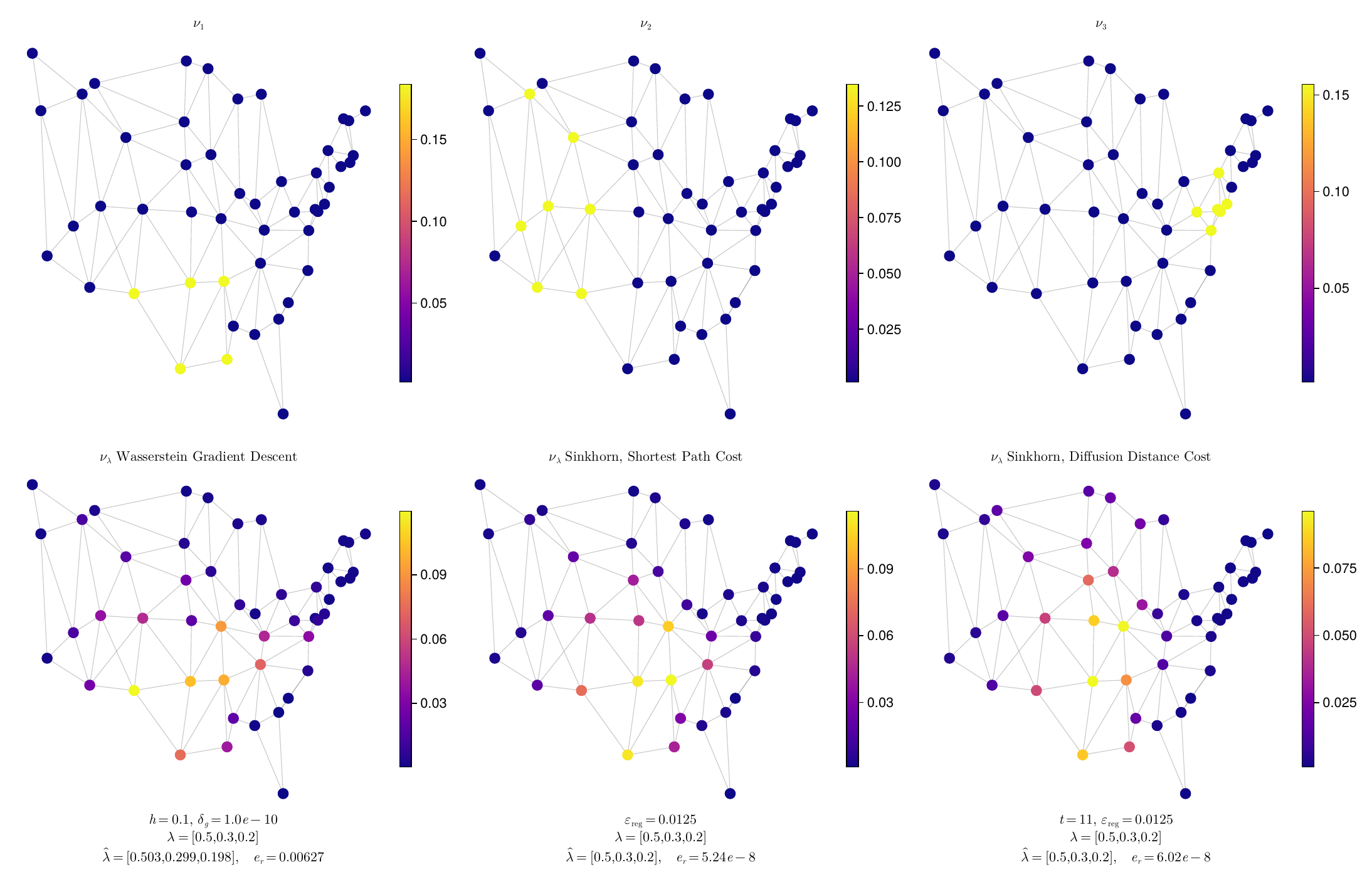}
  \caption{Barycenter of three geographically concentrated measures on a graph representing the 48 contiguous states + DC. Bottom left: computed via intrinsic gradient descent and the dynamic formulation of discrete transport; bottom center: computed via entropic regularization with cost given by the squared shortest path metric; bottom right: computed via entropically regularized transport, with the cost matrix given by the squared diffusion distance for \(t=11\). Parameters used to compute barycenters: \(\varepsilon = 0.1, N=3, \delta_g = 10^{-12}, \delta_b = 10^{-10}\).}
  \label{fig:StatesEntropicComparison}
\end{figure}

\subsection{Effect of Different Initializations for the Descent Scheme}
Lacking guarantees about the geodesic convexity of \eqref{GraphVF}, or its strict convexity with respect to the Euclidean geometry, a natural question to ask is whether or not the choice of initialization \(\nu_0\) in the descent scheme \eqref{step} has a meaningful effect on the convergence behavior of the algorithm, and whether or not it becomes stuck in local minima. To test this, we generate a collection of arbitrary measures on a graph of \(25\) vertices, as well as a spatial geography data graph representing the 48 contiguous states of the U.S., plus D.C., and initiate the descent from each of the reference measures. We make this choice because, as prescribed in \cite{afsariConvergenceGradientDescent2013}, it is generally best practice to choose an initialization we know to be within the geodesic hull of the dictionary of reference measures.  Results are displayed in Figures \ref{fig:GridInit} and \ref{fig:StatesInit}. Once the algorithm converges for each choice of initialization, we compare the distance between the output barycenters in the \(\|\cdot\|_\pi\) norm. Barycenters are considered to be converged at a threshold of \(\delta_b = 10^{-10}\); for the experiment on the grid, the maximal norm difference between computed barycenters was observed to be \(10^{-9}\), while on the graph of the contiguous states the maximal observed norm difference was \(8\times 10^{-10}\), which indicates that the choice of initializing measure for the descent scheme may not have a meaningful impact on the final output of the algorithm. This also suggests a modest computational optimization to make: initiating descent from the dictionary atom
\[\argmin_{i}\left\{\J[\nu_i]\right\}.\]

\begin{figure}
    \centering
    \includegraphics[width=0.95\linewidth]{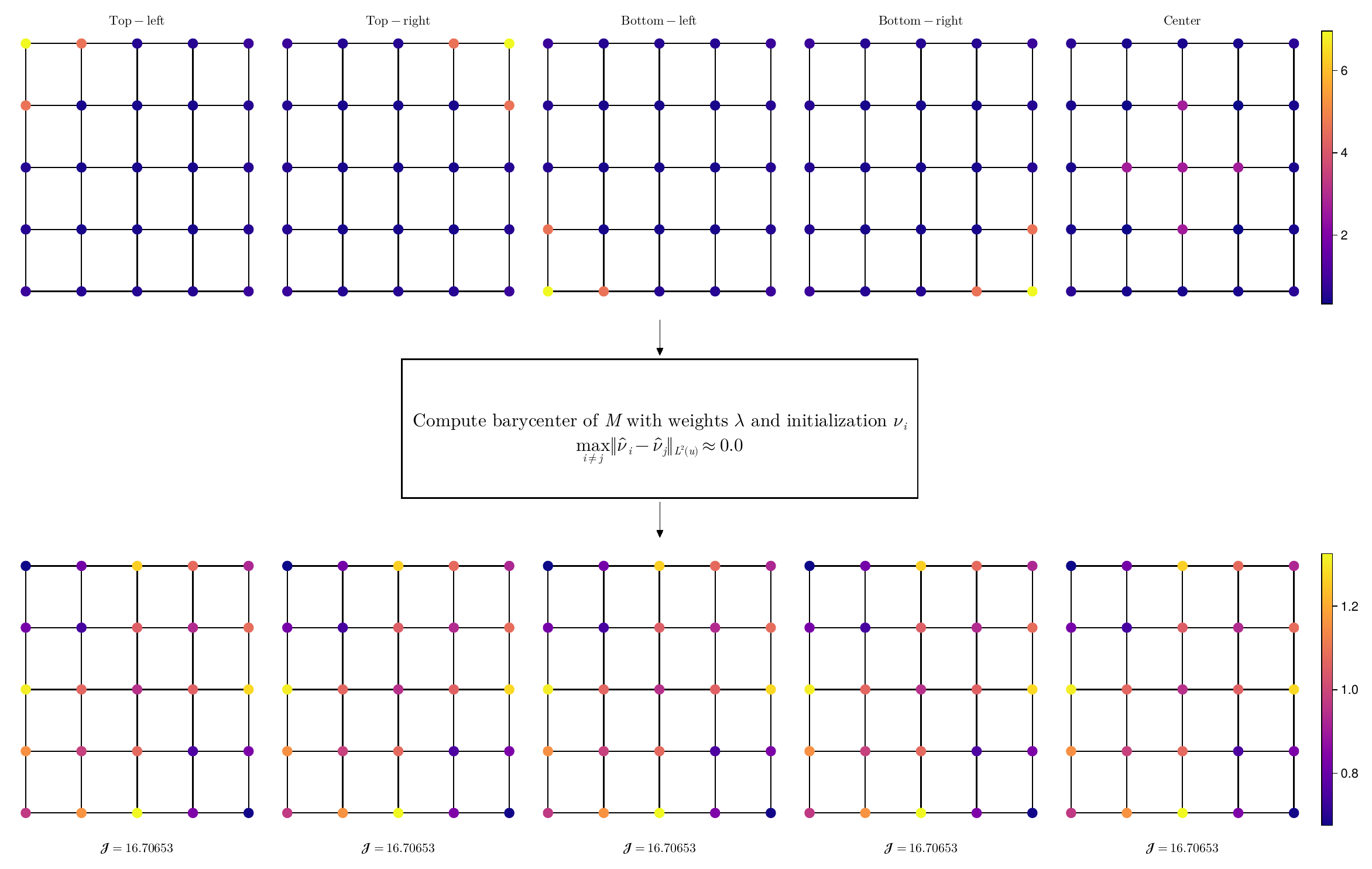}
    \caption{Barycenter of several reference measures computed with differing initializations for the descent scheme. The maximum observed difference in \(\|\cdot\|_\pi\) between the computed barycenters was \(1\times10^{-9}\), the parameters used to compute the barycenters were \(N=16, \delta_g=10^{-6}, \varepsilon=0.25, \delta_b=10^{-10}.\)}
    \label{fig:GridInit}
\end{figure}

\begin{figure}
    \centering
    \includegraphics[width=0.95\linewidth]{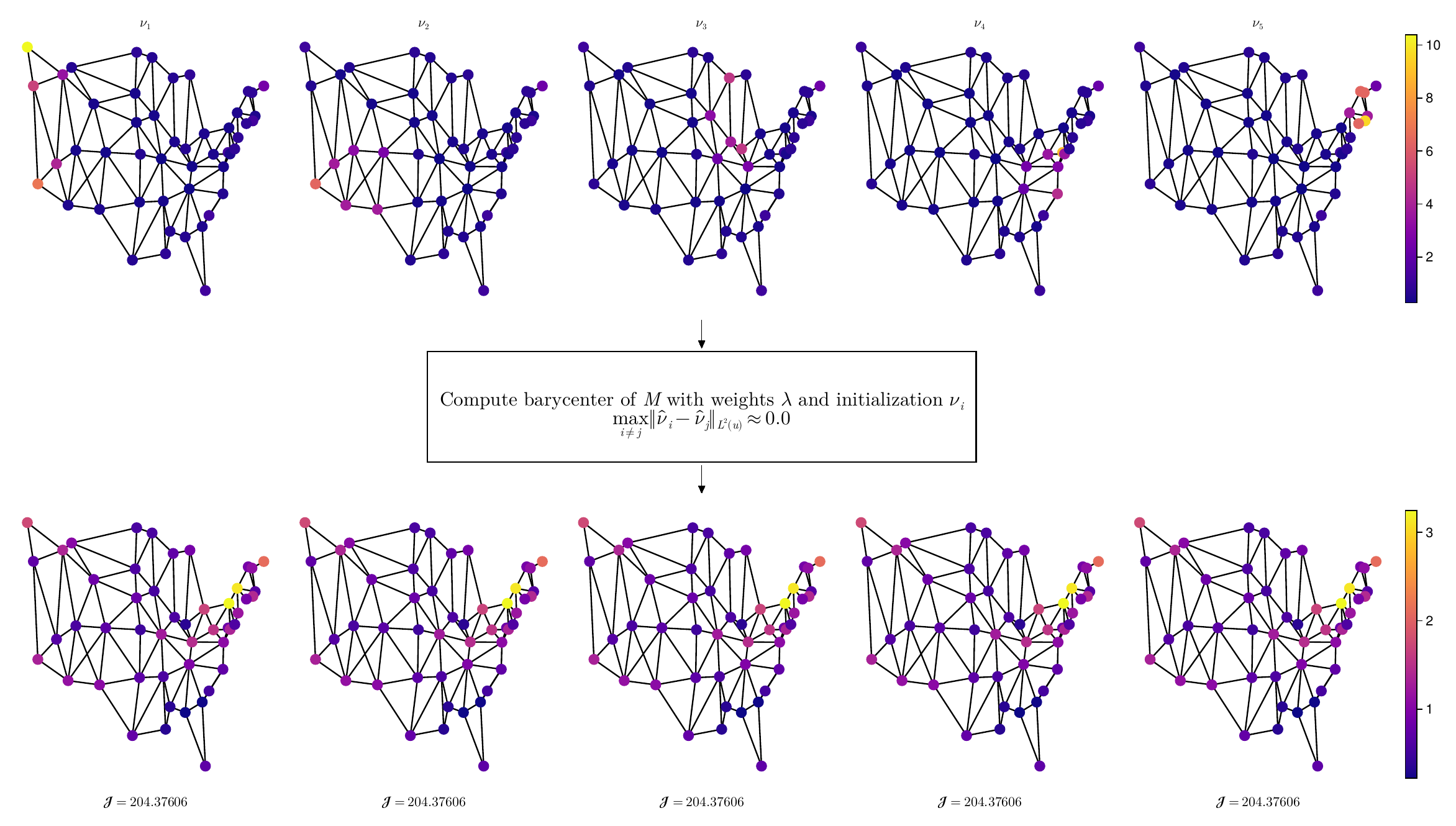}
    \caption{Barycenter of several reference measures computed with differing initializations for the descent scheme. The maximum observed difference in \(\|\cdot\|_\pi\) between the computed barycenters was \(8\times10^{-10}\), the parameters used to compute the barycenters were \(N=5, \delta_g=10^{-6}, \varepsilon=0.25, \delta_b=10^{-10}.\)}
    \label{fig:StatesInit}
\end{figure}

\subsection{Recovery of Coordinates and Effect of Hyperparameters on Recovery Quality}
Perhaps the single most desirable property this tool could have is consistency. It should be the case that if we fix \(\lambda \in \Delta^{p-1}\) and a family of reference measures \(R = \{\nu_i\}_{i=1}^p\), then compute their corresponding barycenter \(\nu^*\), and finally analyze the computed \(\nu^*\) with respect to the family of reference measures \(R\) to obtain coordinates \(\lambda^*\), then it should be the case that \(\lambda \approx \lambda^*\).  We test the consistency of our algorithm in two ways.  For the first experiment we fixed a grid graph on \(16\) nodes, generated \(3\) random measures, a random set of coordinates \(\lambda \in \Delta^2\), and computed the relative error of the recovered coordinates \(\lambda^*\), taking \(\delta_g = 10^{-9}, N=10, \varepsilon=0.5, \delta_b=10^{-10}\). We find that a large majority of experiments result in \(< 0.01\) relative error, and refer to Figure \ref{fig:GridRecoveryRand} for the full distribution of relative error in coordinate recovery. For the second experiment, we modulate the hyperparameters associated to the Chambolle-Pock routine to test the effect of geodesic quality on coordinate recovery. For this experiment, we fix a single graph on \(64\) nodes, three measures \(\nu_1, \nu_2, \nu_3\) and set of coordinates \(\lambda \in \Delta^2\), then vary the hyperparameters \(\delta_g\) and \(N\) to test their effect on coordinate recovery. Our results are optimistic: even taking very coarse parameters, we see the relative error stays under \(0.01\) (see \ref{fig:ParamVariations} for details). This indicates that the computational cost of synthesizing barycenters might be significantly reduced by a judicious choice of hyperparameters.
\begin{figure}
    \centering
    \includegraphics[width=0.6\linewidth]{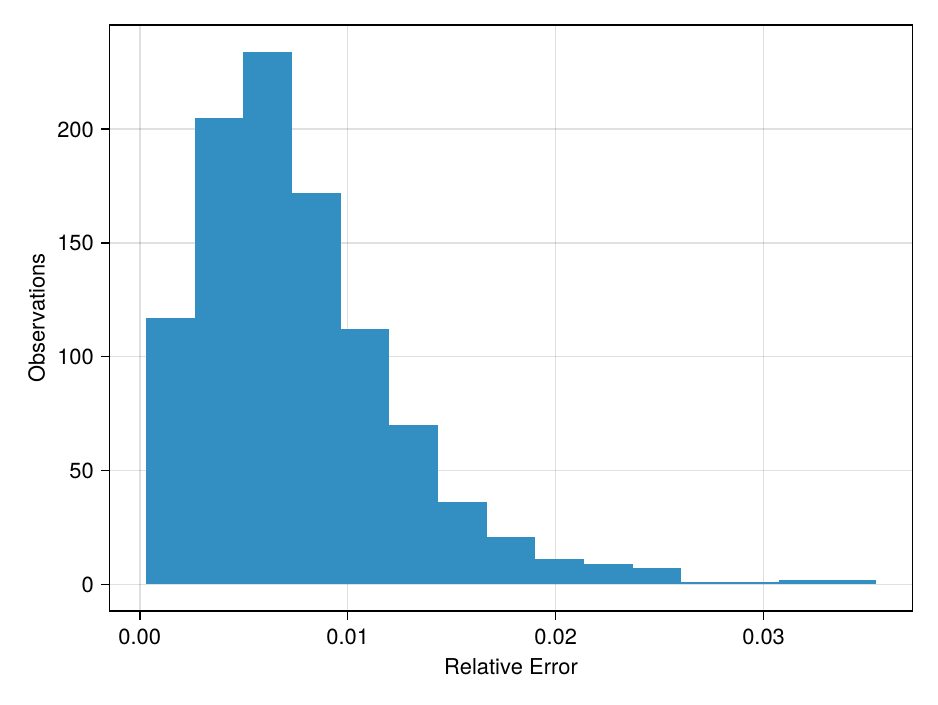}
      \caption{Distribution of 1000 observed relative errors in recovery of coordinates from a barycenter synthesized via intrinsic gradient descent for a grid graph on \(16\) nodes, \(3\) random measures, a random set of coordinates \(\lambda \in \Delta^2\), with hyperparameters \(\delta_g = 10^{-9}, N=10, \varepsilon=0.5, \delta_b=10^{-10}\).}
    \label{fig:GridRecoveryRand}
\end{figure}

\begin{figure}[htbp]
\begin{subfigure}[t]{0.48\textwidth}
  \centering
  \includegraphics[width=\linewidth]{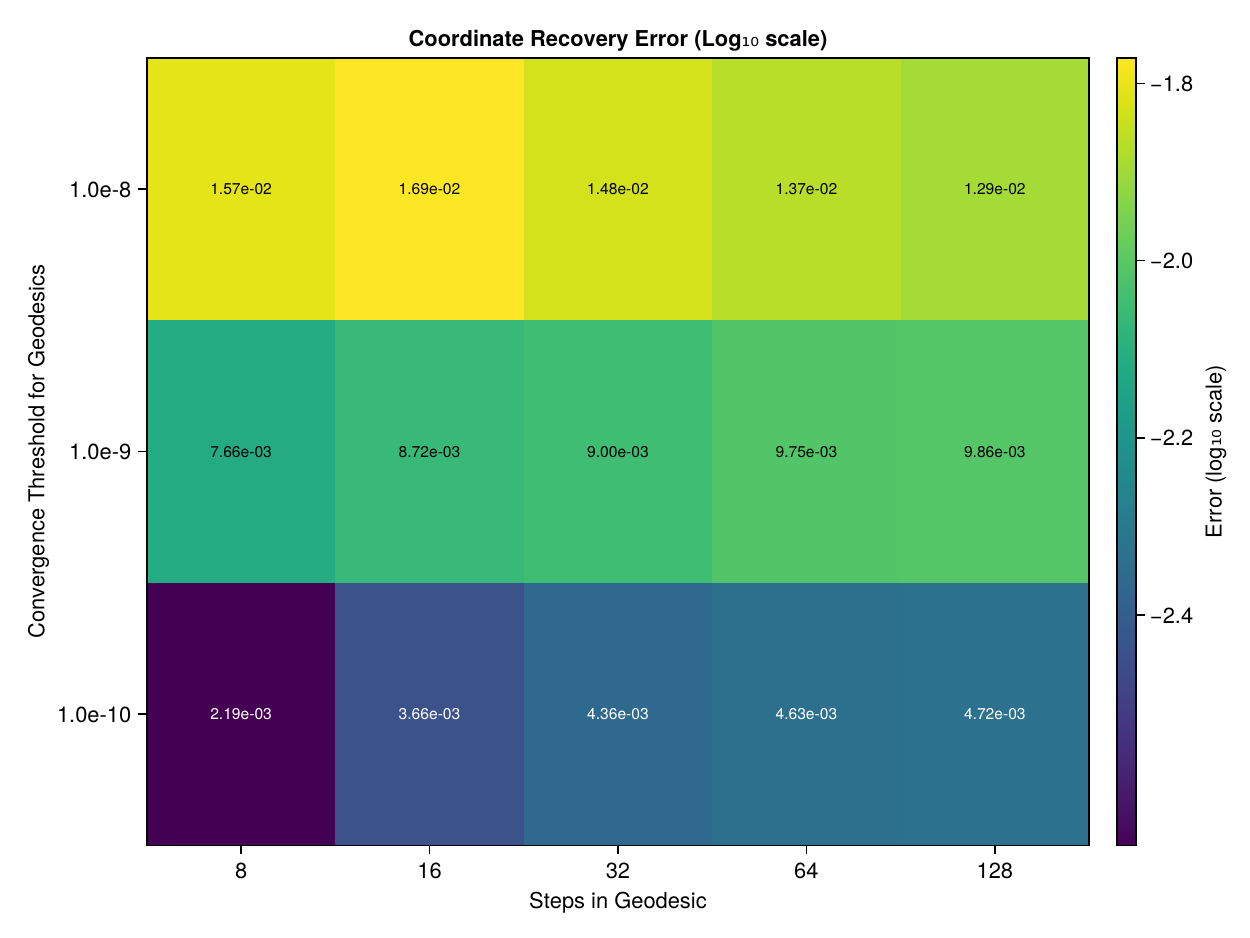}
  \caption{Error matrix for coordinate recovery experiment. A graph defined by the \(48\) contiguous states + D.C. is fixed and three measures \(\nu_1, \nu_2, \nu_3\) are generated randomly, along with randomly chosen coordinates \(\lambda\). The barycenter \(\nu_\lambda\) is computed via intrinsic gradient descent, and then analyzed by solving the associated quadratic program. The relative error of the recovered coordinates is recorded as a function of the convergence threshold for the geodesic computations and the number of steps in each geodesic. Convergence threshold for the barycenters is set to \(\delta_g = 10^{-8}\), and stepsize is taken to be \(\varepsilon = 0.05\).}
\end{subfigure}
 \hfill 
  \begin{subfigure}[t]{0.48\textwidth}
  \centering
  \centering
  \includegraphics[width=\linewidth]{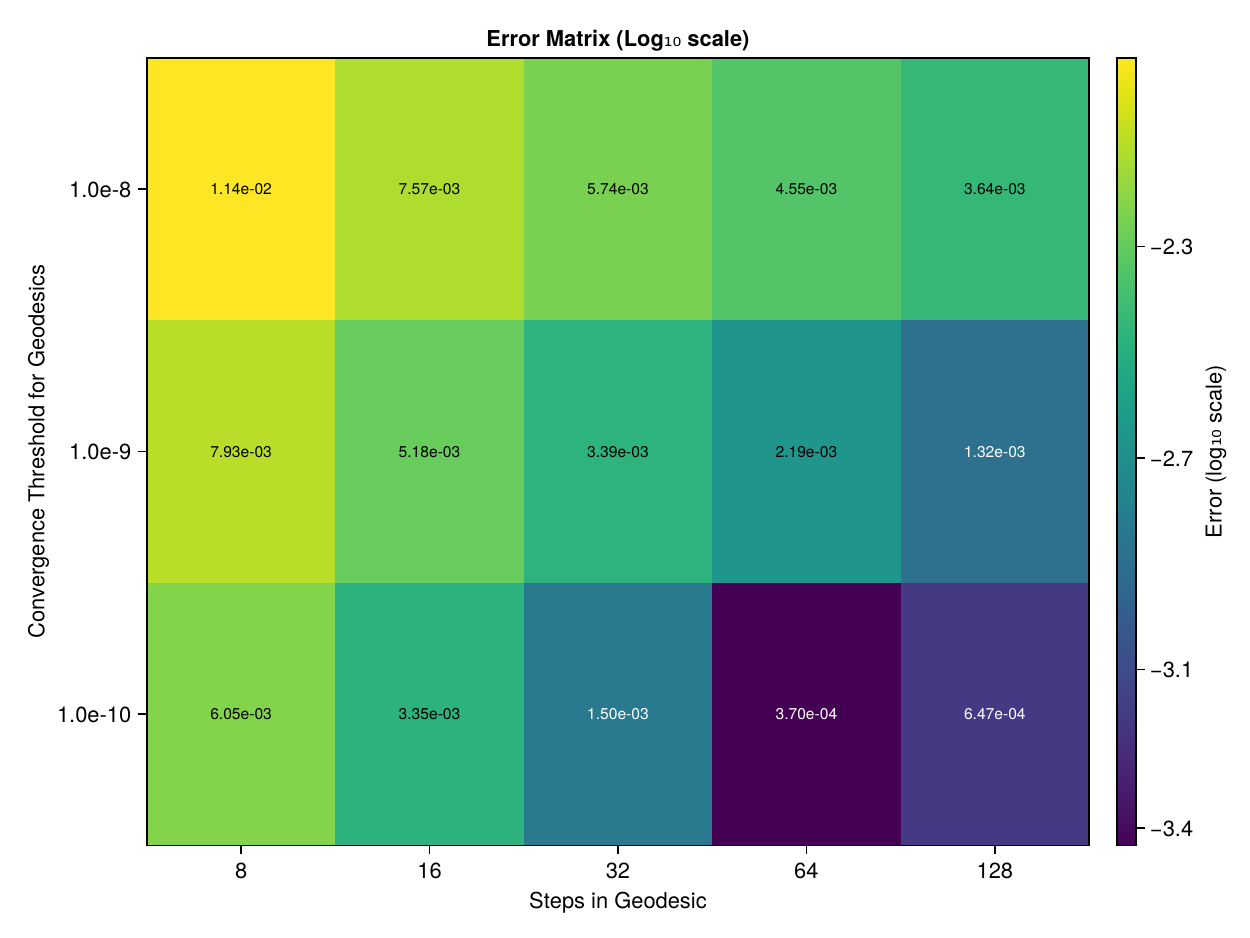}
  \caption{Error matrix for coordinate recovery experiment. A grid graph on \(64\) vertices is fixed and three measures \(\nu_1, \nu_2, \nu_3\) are generated randomly, along with randomly chosen coordinates \(\lambda\). The barycenter \(\nu_\lambda\) is computed via intrinsic gradient descent, and then analyzed by solving the associated quadratic program. The relative error of the recovered coordinates is recorded as a function of the convergence threshold for the geodesic computations and the number of steps in each geodesic. Convergence threshold for the barycenters is set to \(\delta_g = 10^{-8}\), and stepsize is taken to be \(\varepsilon = 0.05\).}
  \end{subfigure}
  \caption{Recovery of coordinates for synthesized barycenters.}
\label{fig:ParamVariations}
\end{figure}
\subsection{Empirical Convergence}
Lacking analytical information about either the injectivity radius or the sectional curvatures of the space \(\mathcal{P}(\X)\), we cannot make guarantees about the rate of convergence for the intrinsic gradient descent scheme \cite{afsariConvergenceGradientDescent2013}, which is unfortunate. However, we are able to study the rate of convergence for our algorithm empirically. For this experiment, we fixed a grid graph on \(16\) vertices, and for each run of the experiment, generated \(3\) random measures and a random set of barycentric coordinates \(\lambda\); then we tracked the difference in norm between iterates as the descent scheme proceeded. In other words, we track the quantity
\begin{equation*}
    \|\nu_{k+1} - \nu_{k}\|_\pi = h \|\divx\nabla J|_{\nu_k}\|_{\pi}.
\end{equation*}
We repeated this procedure 100 times, with a convergence threshold \(\delta_{b} = 10^{-10}\), and plotted the results. We observe clear linear convergence behavior in Figure \ref{fig:LinearConvergence}. 
\section{Discussion}
\begin{figure}
    \centering
    \includegraphics[width=0.78\linewidth]{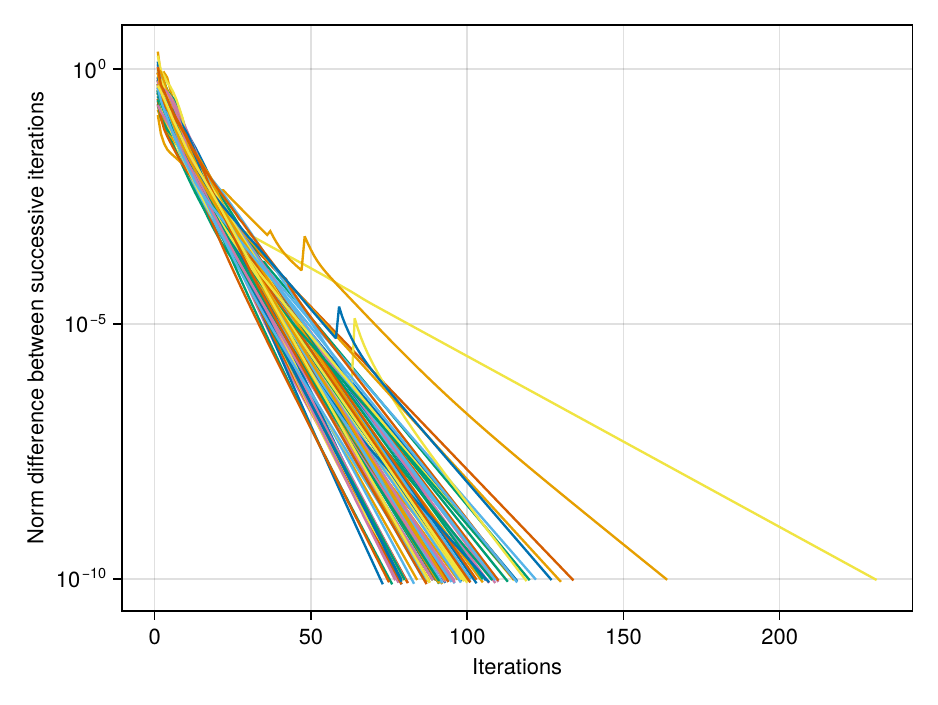}
    \caption{Log-linear scale of the absolute \(\ell^2\) norm difference between consecutive iterates and the number of iterations, or equivalently, the norm of the \(\W\) gradient at each iteration multiplied by the stepsize \(\varepsilon\). Hyperparameters for this experiment: \(\varepsilon =0.25\), \(\delta_b = 10^{-10}\), \(N=10\), \(\delta_g = 10^{-9}\).}
  \label{fig:LinearConvergence}
\end{figure}

We have built and made available a novel signal processing tool for probability measures on graphs. We contrasted a dynamic approach to synthesis of barycenters for measures supported on a graph via intrinsic gradient descent on the simplex to a static approach using the Sinkhorn algorithm. While Sinkhorn produces synthesized barycenters more quickly, the intrinsic gradient descent approach does not rely on any particular choice being made to encode the connectivity of the graph --- it is enough to work just with the Markov chain representation of the space.

Much remains to be done. This work is essentially computational and experimental, and many pressing theoretical questions remain to be investigated. We list just some of them here:
\begin{enumerate}
  \item Although Lemma \ref{lemma1} shows that the gradient descent iterates in the synthesis algorithm remain signed measures of unit mass with respect to the stationary distribution \(\pi\), ultimately we would like to find sufficient conditions guaranteeing that the iterates are always non-negative probability measures.
  \item While it is clear that barycenters exist by virtue of the convexity of the variance functional with respect to linear interpolation, it is \emph{not} clear that the variance functional is itself geodesically convex. In the case of measures supported on a Euclidean domain, the answer to the corresponding question is known, and in fact it is known to not be geodescially convex --- nevertheless, gradient descent in the \(W_2\) metric has been shown to be a legitimate tool for approximation of \(W_2\) barycenters under suitable conditions \cite{chewiGradientDescentAlgorithms2020}. It is of great interest to decide if variance functional is geodesically convex in our discrete setting.
  \item Can we make meaningful improvements to the algorithm described in \cite{erbarComputationOptimalTransport2020}?
  \item Given the popularity of entropic regularization in the computational approximation of Euclidean transport, can such a regularization be found for discrete transport distances?
  \item The step-size in the intrinsic gradient descent scheme proposed in this paper is taken as a constant for simplicity, but adaptive step sizes often have computational benefits in achieving faster rates of convergence \cite{afsariConvergenceGradientDescent2013, chambolleFirstOrderPrimalDualAlgorithm2011} --- could such a scheme be appropriate and beneficial in this setting?
  \item Alternatively, changing nothing at all about the way that we are currently approaching the problem of synthesizing and analyzing barycenters, there are already hyperparameters which have an unclear effect on the quality of our final outputs, namely the number of steps used to compute discrete geodesics, the convergence tolerance for the Chambolle-Pock routine, and the convergence tolerance for the intrinsic gradient descent scheme. It is entirely conceivable that a high tuning any of these knobs (e.g., taking the number of steps per geodesic \(N\) to be large) will exhibit diminishing returns.
  \item It has been shown in \cite{erbarGeometryGeodesicsDiscrete2019} that the geodesics of the \(\W\) metric may exhibit non-locality if certain topological properties are not satisfied. It would be of great interest to see if similar results can be extended to \(\W\) barycenters --- i.e., if \(\{\nu_i\}_i\) is a collection of measures supported on a subset \(A \subset \X\), when can it be guaranteed that the barycenter \(\Bary(\{\nu_i\}_i)\) is also supported entirely in \(A\)?
  \item We have totally avoided consideration of the boundary of the simplex in our discussion, and it is known that geodesic curves in the \(\W\) geometry may intersect the boundary of \(\mathcal{P}(\X)\) even when both end points lie in the interior of \(\mathcal{P}(\X)\) \cite{gangboGeodesicsMinimalLength2019}. 
  \item One may consider warm starts as follows.  Fix a family of reference measures \(\{\nu_i\}_{i=1}^p\) and let \(\nu^{(k)}\) be the \(k\)-th iterate of the descent scheme for the synthesis of \(\Bary(\left\{\nu_i\right\}, \lambda)\) for some arbitrary \(\lambda\). Denote by \(\rho_i^{(k)}\) the (rectified) geodesic curves computed by the Chambolle-Pock routine which connect \(\nu^{(k)}\) to \(\nu_i\). Presuming that \(\|\nu^{(k+1)} - \nu^{(k)}\|_{\pi}\) is small, it seems reasonable to conclude that \(\|\rho_i^{(k+1)} - \rho_i^{(k)}\|_H\) will also be small. In that case, converged primal and dual variables of the Chambolle-Pock routine obtained on the \(k\)-th iteration of the synthesis descent scheme may prove to be suitable initializations for the \((k+1)\)-th iteration, and yield faster convergence for our algorithm.
\end{enumerate}

\vspace{10pt}

\noindent \textbf{Acknowledgments:}  This work was partially supported by the National Science Foundation through grants DMS-2309519 and DMS-2318894. The authors acknowledge the use of Claude Code for assistance with code documentation, generating code to produce figures, and a significant improvement to one of the Newton methods implemented as a part of the Chambolle-Pock routine.

\bibliography{library}
\end{document}